\documentclass[]{onecat}

\usepackage[utf8]{inputenc}
\usepackage{caption}
\usepackage{hyperref}
\usepackage{cleveref}
\usepackage{verbatim}
\usepackage{amsmath}
\usepackage{amssymb}
\usepackage{amsfonts}
\usepackage{amsthm}
\usepackage{mathtools}
\usepackage{graphicx}
\usepackage{subcaption}
\usepackage{listings}
\usepackage{algorithm}
\usepackage{makecell}
\usepackage{multirow}
\usepackage{tabularx}
\usepackage{array}
\usepackage{textcomp}
\usepackage{enumitem}
\usepackage{wrapfig}
\usepackage{tikz}
\usepackage{marvosym}
\usepackage{pifont}
\usepackage{minitoc}
\newcommand{\dflab}{D\textsuperscript{4} Lab}

\usetikzlibrary{arrows.meta,positioning,fit,calc}

\definecolor{PromptFrame}{HTML}{C99700}
\definecolor{PromptTitle}{HTML}{D9A600}
\definecolor{PromptBack}{HTML}{FFF4BF}
\definecolor{SpecFrame}{HTML}{4F7CAC}
\definecolor{SpecTitle}{HTML}{D9EAF7}
\definecolor{SpecBack}{HTML}{F5FAFF}
\definecolor{SchemaFrame}{HTML}{5F6B7A}
\definecolor{SchemaBack}{HTML}{F7F8FA}
\definecolor{CaseFrame}{HTML}{6B7280}
\definecolor{CaseBack}{HTML}{F8FAFC}
\definecolor{CaseTitle}{HTML}{E5E7EB}
\definecolor{GoodBack}{HTML}{ECFDF5}
\definecolor{BadBack}{HTML}{FEF2F2}
\definecolor{NeutralBack}{HTML}{F9FAFB}
\definecolor{llwinbg}{RGB}{226,245,232}
\definecolor{lllossbg}{RGB}{252,228,228}
\definecolor{llneutralbg}{RGB}{245,245,245}
\definecolor{basegray}{RGB}{247,247,247}
\definecolor{droporange}{RGB}{255,236,220}
\definecolor{recoverteal}{RGB}{225,243,239}
\definecolor{qualgreen}{RGB}{225,243,239}
\definecolor{qualred}{RGB}{255,236,220}
\definecolor{qualgray}{RGB}{244,244,244}

\newcommand{\llwin}[1]{\cellcolor{llwinbg}{#1}}
\newcommand{\llloss}[1]{\cellcolor{lllossbg}{#1}}
\newcommand{\llneutral}[1]{\cellcolor{llneutralbg}{#1}}

\providecommand{\qpos}[1]{\cellcolor{qualgreen}\textbf{#1}}

\providecommand{\qnear}[1]{\cellcolor{qualgray}{#1}}

\newcommand{\pp}{\,\text{pp}}

\NewDocumentCommand{\tyh}{ mO{} }{\textcolor{blue}{\textsuperscript{\textit{tyh}}\textsf{\textbf{\small[#1]}}}}
\NewDocumentCommand{\shy}{ mO{} }{\textcolor{red}{\textsuperscript{\textit{shy}}\textsf{\textbf{\small[#1]}}}}

\newcolumntype{Y}{>{\raggedright\arraybackslash}X}

\theoremstyle{definition}
\newtheorem{assumption}{Assumption}[section]
\newtheorem{lemma}{Lemma}[section]
\newtheorem{proposition}{Proposition}[section]

\newtheorem{corollary}{Corollary}[section]

\lstdefinestyle{promptstyle}{
  basicstyle=\ttfamily\scriptsize,
  breaklines=true,
  breakatwhitespace=false,
  columns=fullflexible,
  keepspaces=true,
  showstringspaces=false,
  tabsize=2,
  upquote=true
}

\lstdefinelanguage{json}{
  basicstyle=\ttfamily\small,
  breaklines=true,
  morecomment=[l]{//},
  morestring=[b]",
  stringstyle=\color{red!70!black},
  commentstyle=\color{gray!70!white},
  literate=
    *{0}{{{\color{gray}{0}}}}1
     {1}{{{\color{gray}{1}}}}1
     {:}{{{\color{black}{:}}}}1
     {,}{{{\color{black}{,}}}}1
     {\{}{{{\color{blue}{\{}}}}1
     {\}}{{{\color{blue}{\}}}}}1
     {[}{{{\color{blue}{[}}}}1
     {]}{{{\color{blue}{]}}}}1
}

\tcbuselibrary{listingsutf8}

\newtcolorbox{casebox}[2][]{%
  enhanced,
  breakable,
  sharp corners,
  colback=CaseBack,
  colframe=CaseFrame,
  colbacktitle=CaseTitle,
  coltitle=black,
  fonttitle=\bfseries,
  title={#2},
  left=1.2mm,
  right=1.2mm,
  top=1mm,
  bottom=1mm,
  #1
}

\newtcblisting{promptlisting}[2][]{%
  enhanced,
  breakable,
  sharp corners,
  colback=PromptBack,
  colframe=PromptFrame,
  colbacktitle=PromptTitle,
  coltitle=black,
  fonttitle=\bfseries,
  title={#2},
  listing only,
  listing options={style=promptstyle},
  left=1mm,
  right=1mm,
  top=1mm,
  bottom=1mm,
  #1
}

\newtcolorbox{specbox}[2][]{%
  enhanced,
  breakable,
  sharp corners,
  colback=SpecBack,
  colframe=SpecFrame,
  colbacktitle=SpecTitle,
  coltitle=black,
  fonttitle=\bfseries,
  title={#2},
  left=1.2mm,
  right=1.2mm,
  top=1mm,
  bottom=1mm,
  #1
}

\newtcblisting{schemalisting}[2][]{%
  enhanced,
  breakable,
  sharp corners,
  colback=SchemaBack,
  colframe=SchemaFrame,
  colbacktitle=SchemaFrame!20,
  coltitle=black,
  fonttitle=\bfseries,
  title={#2},
  listing only,
  listing options={style=promptstyle},
  left=1mm,
  right=1mm,
  top=1mm,
  bottom=1mm,
  #1
}

\title{Rethinking Local Learning: A Cheaper and Faster Recipe for LLM Post-Training}

\author{%
\parbox{\textwidth}{\centering
Hengyu Shi$^{1*}$, Tianyang Han$^{2*}$, Peizhe Wang$^{3*}$ \\ [0.5em]
Zhiling Wang$^{1\ddagger}$, Xu Yang$^{2}$, Junhao Su$^{1\ddagger\dagger}$\\
}
}

\affiliation{%
\parbox{\textwidth}{\centering\small
$^1$Independent Researcher, \quad $^2$\dflab, \quad $^3$Southeast University
}}

\contribution[*]{Equal contribution}
\contribution[\ddagger]{Corresponding author}
\contribution[\dagger]{Project leader}

\abstract{
LLM post-training typically propagates task gradients through the full depth of the model. Although this end-to-end structure is simple and general, it couples task adaptation to full-depth activation storage, long-range backward dependencies and direct task-gradient access to pretrained representations. We argue that this full-depth backward coupling can be unnecessarily expensive and intrusive, particularly when post-training supervision is much narrower than pre-training. To this end, we propose \textbf{LoPT}: \textbf{Lo}cal-Learning \textbf{P}ost-\textbf{T}raining, a simple post-training strategy that makes gradient reach an explicit design choice.
LoPT places a single gradient boundary at the transformer midpoint: the second-half block learns from the task objective, while the first-half block is updated by a lightweight feature-reconstruction objective to preserve useful representations and maintain interface compatibility. LoPT shortens the task-induced backward path while limiting direct interference from narrow task gradients on early-layer representations. Extensive experiments demonstrate that LoPT achieves competitive performance with lower memory cost, higher training efficiency and better retention of pretrained capabilities. Our code is available at: \url{https://github.com/HumyuShi/LoPT}
}

\date{\today}

\begin{document}

\maketitle

\section{Introduction}
\label{sec:introduction}

Post-training has become the primary route for adapting pretrained large language models (LLMs) into practically useful systems. Methods such as supervised fine-tuning (SFT), direct preference optimization (DPO) and recent reinforcement-learning variants substantially shape model behavior after pre-training~\citep{ouyang2022training,rafailov2024direct,shao2024deepseekmath}. Although these approaches differ in their supervision signals, training objectives, and optimization procedures, they often share an underlying optimization pattern in which task gradients are propagated end-to-end through the full depth of the model. The resulting full-depth gradient pathway is simple and general, but it is usually treated as a fixed implementation default rather than an explicit design choice in post-training.


Prior work has mainly focused on which objectives should be optimized, which parameters should be updated and how training cost can be reduced through systems and optimization techniques. While these directions are important, they do not fully characterize the computational burden of post-training. Even when updates are restricted to a subset of parameters, the task loss may still induce a backward graph that spans the full transformer stack. Post-training in this setting must therefore retain activations across layers, maintain long-range backward dependencies and assign credit through the entire model depth. The cost and behavior of post-training are thus governed not only by the number of trainable parameters but also by the reach of task gradients, which we identified as how far the task loss is allowed to propagate through the model.

The importance of gradient reach is amplified in post-training, where supervision is often much narrower than the data and objectives used during pre-training. 
A post-training dataset may emphasize a particular response format, task family, or reward signal, while standard end-to-end training still allows the resulting task gradients to propagate through the entire model. This leads to full-depth backward coupling, in which task adaptation remains tied to full-depth activation storage, long-range credit assignment, and direct gradient updates to representations learned from broad pre-training data. 
Such coupling can be unnecessarily costly when the desired behavioral change is concentrated in task-facing layers, and unnecessarily intrusive when narrow task gradients perturb lower- and middle-layer representations that support general capabilities~\citep{han2025beyond}.

\begin{figure}[t]
    \centering
    \includegraphics[width=\linewidth]{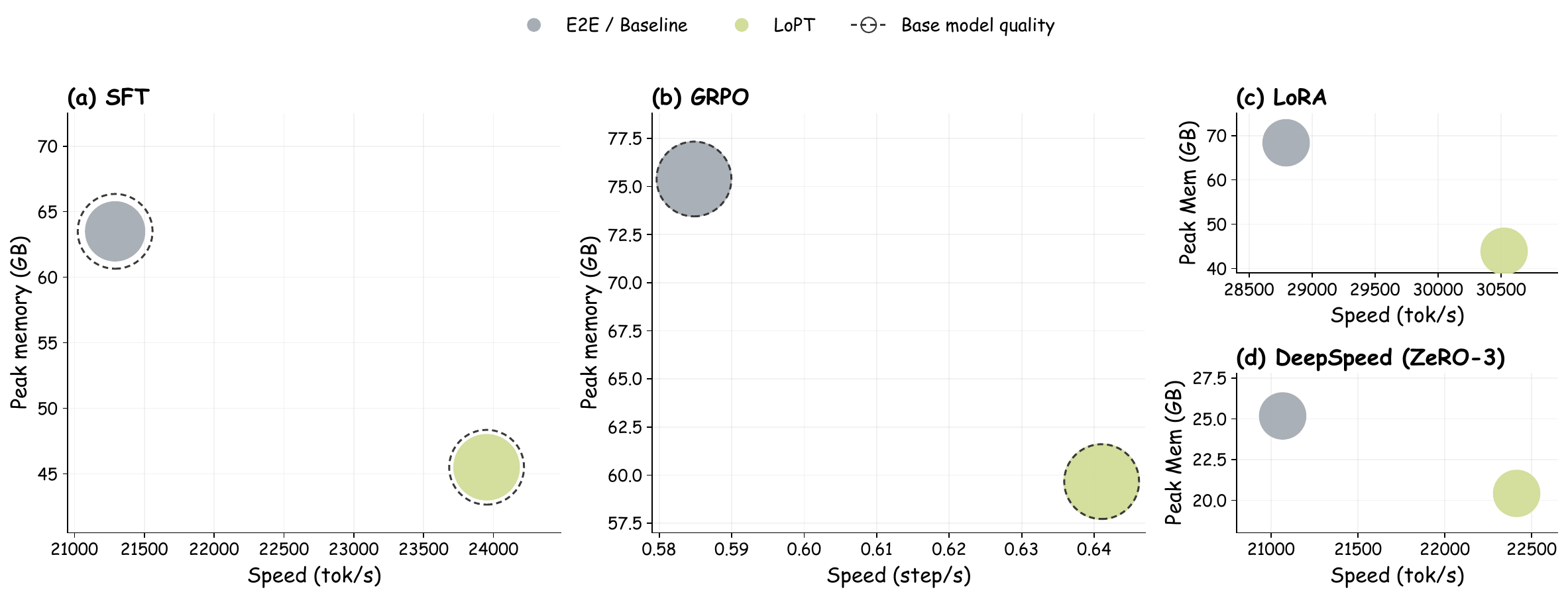}
    \caption{\textbf{LoPT improves the quality--efficiency trade-off in post-training.}
    Panels (a) and (b) report Qwen2.5-7B-Instruct results for SFT and GRPO. Bubble size denotes the seven-benchmark average quality, and the dashed ring denotes base-model quality.
    LoPT reduces peak memory and improves training speed while improving SFT quality and maintaining near-parity GRPO quality.
    Panels (c) and (d) report matched-stack LoRA and DeepSpeed-ZeRO3 profiling in absolute throughput and peak-memory units, matching the appendix tables.
    Quality values are averaged over three independent runs with different random seeds, and efficiency values are averaged over three measured runs.}
    \label{fig:teaser}
\end{figure}


Supervised local-learning methods offer a potential solution to this dilemma by introducing local auxiliary objectives that shorten the backpropagation path, thereby mitigating disruption to the model’s shallow-layer features ~\citep{su2024momentum,zhang2025mlaan,su2026manpp}. 
Motivated by the need to control task-gradient reach during LLM post-training, we introduce \textbf{Local-Learning Post-Training (LoPT)}, a simple local-learning-style post-training scheme for LLMs. LoPT partitions the transformer into a first-half block \(k_1\) and a second-half block \(k_2\) by inserting a single \texttt{detach()} boundary at the midpoint. The second-half block is optimized with the standard task objective, whereas the first-half block is shielded from direct task-gradient propagation and is trained via gradients obtained from a lightweight feature-reconstruction objective. The local objective does not aim to solve the downstream task independently. Instead, it keeps the first-half block trainable, encouraging it to preserve useful pretrained representations while maintaining compatibility with the adapting second-half block. Therefore, LoPT focuses on controlling the reach of task gradients rather than altering the number of trainable parameters. While LoRA-style methods modify the parameterization of adaptation, LoPT instead controls how far task gradients are allowed to propagate through the model.

Specifically, by stopping task-gradient propagation at the midpoint, LoPT shortens the backward path induced by the post-training objective, reducing the need to retain activations for backpropagation and weakening long-range backward dependencies. This allows the first-half block to remain trainable through its local objective, offering more flexibility than simply freezing early layers while limiting interference from narrow task gradients. As a result, LoPT decouples task-facing adaptation in the second-half block from the preservation of lower-level representations maintained through the first-half block. The second-half block specializes in the post-training objective, while the first-half block is encouraged to preserve a stable and useful interface. Importantly, gradient isolation works orthogonally to parameter-efficient tuning and system optimizations, and it can be naturally combined with techniques such as LoRA, gradient checkpointing, DeepSpeed/ZeRO, and pipeline parallelism. Figure~\ref{fig:teaser} summarizes the resulting quality--efficiency trade-off.


Our main contributions are summarized as follows.
\begin{itemize}
\item We revisit the default pattern of full-depth task-gradient propagation in LLM post-training and formulate gradient reach as an underexplored design variable. By focusing on how far task gradients propagate, we connect training memory, optimization complexity, and interference with pretrained representations as different consequences of full-depth backward coupling.

\item We propose \textbf{Local-Learning Post-Training (LoPT)}, a simple post-training scheme that localizes task-gradient propagation with a single midpoint boundary. LoPT optimizes the second-half block with the task objective while keeping the first-half block trainable through a lightweight feature-reconstruction objective.

\item Extensive experimental results demonstrate that LoPT delivers competitive performance across both SFT and GRPO post-training regimes. Beyond task performance, LoPT reduces training memory and improves training efficiency with better preservation of pretrained capabilities.
\end{itemize}

We discuss related work on local learning and memory-efficient LLM post-training in Appendix~\ref{sec:related}.

\section{Related Work}
\label{sec:related}

\subsection{Local learning.}
Local learning has been studied as an alternative to standard end-to-end backpropagation by replacing a single global objective with local training signals. A central motivation is to reduce backward dependencies and alleviate locking phenomena, where earlier modules must wait for later modules to complete forward and backward computation before being updated. Decoupled neural interfaces and synthetic gradients address this issue by predicting gradient information from local activations~\citep{jaderberg2017decoupled}. Supervised local learning further partitions deep networks into blocks or layers trained with auxiliary objectives, reducing activation storage and enabling more localized updates~\citep{belilovsky2019decoupled,nokland2019training}. 
Recent methods improve local learning with stronger auxiliary losses and better cross-block communication, narrowing the performance gap with end-to-end training~\citep{su2024momentum,su2024hpff,zhang2025mlaan,zhu2024advancing}. Despite these advances, local learning has mostly been studied as a general vision training alternative, while its role as a mechanism for controlling task-gradient reach in LLM post-training remains underexplored.

\subsection{Memory-efficient LLM post-training.}
Memory-efficient post-training has been widely studied through algorithmic changes to how large language models are adapted. Parameter-efficient tuning methods reduce the number of trainable parameters by updating only small task-specific modules or low-rank adapters, including adapters, prefix tuning, prompt tuning, and LoRA~\citep{houlsby2019parameter,li2021prefix,lester2021power,hu2022lora}. Quantized fine-tuning further lowers memory usage by combining low-rank adaptation with low-bit pretrained weights, as in QLoRA~\citep{dettmers2024qlora}. More recent optimizer-efficient methods reduce training memory from the optimization side, for example by projecting gradients into low-rank subspaces or reducing optimizer-state overhead~\citep{zhao2024galore}. These methods mainly reduce memory by changing which parameters are updated, how adaptation is parameterized, or how optimizer states are represented. In contrast, our work concentrates on how far task gradients should be allowed to propagate through the pretrained model during post-training.

\section{Method}
\label{sec:method}

We first formalize the conventional end-to-end post-training setup in Section~\ref{sec:method_prelim}, then introduce the local-learning perspective and the stop-gradient boundary underlying LoPT in Section~\ref{sec:method_local_learning}. In Section~\ref{sec:method_lopt}, we instantiate LoPT for supervised fine-tuning (SFT) and Group Relative Policy Optimization (GRPO).

\subsection{Preliminaries}
\label{sec:method_prelim}

Consider an autoregressive transformer~\citep{vaswani2017attention} with an embedding layer, $N$ transformer layers, a final normalization layer, and an output head. Given an input token sequence $\mathbf{x}$, we denote its embedding sequence by $\mathbf{h}_0=\texttt{Embed}(\mathbf{x})$. We split the transformer at layer $s$ into two consecutive blocks, a first block \(k_1\) and a second block \(k_2\):
\begin{equation}
\begin{aligned}
    k_1 &= \{\text{Layer}_0,\ldots,\text{Layer}_{s-1}\}, \\
    k_2 &= \{\text{Layer}_{s},\ldots,\text{Layer}_{N-1}, \texttt{norm}, \texttt{lm\_head}\}.
\end{aligned}
\label{eq:model_partition}
\end{equation}
Let $\theta_1$ and $\theta_2$ be the parameters of $k_1$ and $k_2$. The forward computation can be written as:
\begin{equation}
    \mathbf{h}_1=f_{k_1}(\mathbf{h}_0;\theta_1),
    \qquad
    \mathbf{y}=f_{k_2}(\mathbf{h}_1;\theta_2),
\end{equation}
where $\mathbf{h}_1$ is the intermediate representation passed from the first block to the second block, and $\mathbf{y}$ denotes the output logits.

Under standard end-to-end post-training, the above split is only notational. The task loss $\mathcal{L}_{\mathrm{task}}$ is backpropagated through the full model by:
\begin{equation}
\begin{aligned}
    \theta_1 &\leftarrow \theta_1-\eta\nabla_{\theta_1}\mathcal{L}_{\mathrm{task}},
    &
    \nabla_{\theta_1}\mathcal{L}_{\mathrm{task}}
    &=
    \frac{\partial \mathcal{L}_{\mathrm{task}}}{\partial \mathbf{y}}
    \frac{\partial \mathbf{y}}{\partial \mathbf{h}_1}
    \frac{\partial \mathbf{h}_1}{\partial \theta_1},
    \\
    \theta_2 &\leftarrow \theta_2-\eta\nabla_{\theta_2}\mathcal{L}_{\mathrm{task}},
    &
    \nabla_{\theta_2}\mathcal{L}_{\mathrm{task}}
    &=
    \frac{\partial \mathcal{L}_{\mathrm{task}}}{\partial \mathbf{y}}
    \frac{\partial \mathbf{y}}{\partial \theta_2}.
\end{aligned}
\label{eq:e2e_update}
\end{equation}

The key property of E2E post-training is therefore that the task objective directly supervises both blocks through a full-depth backward path.

\begin{figure}[t]
    \centering
    \includegraphics[width=0.95\linewidth]{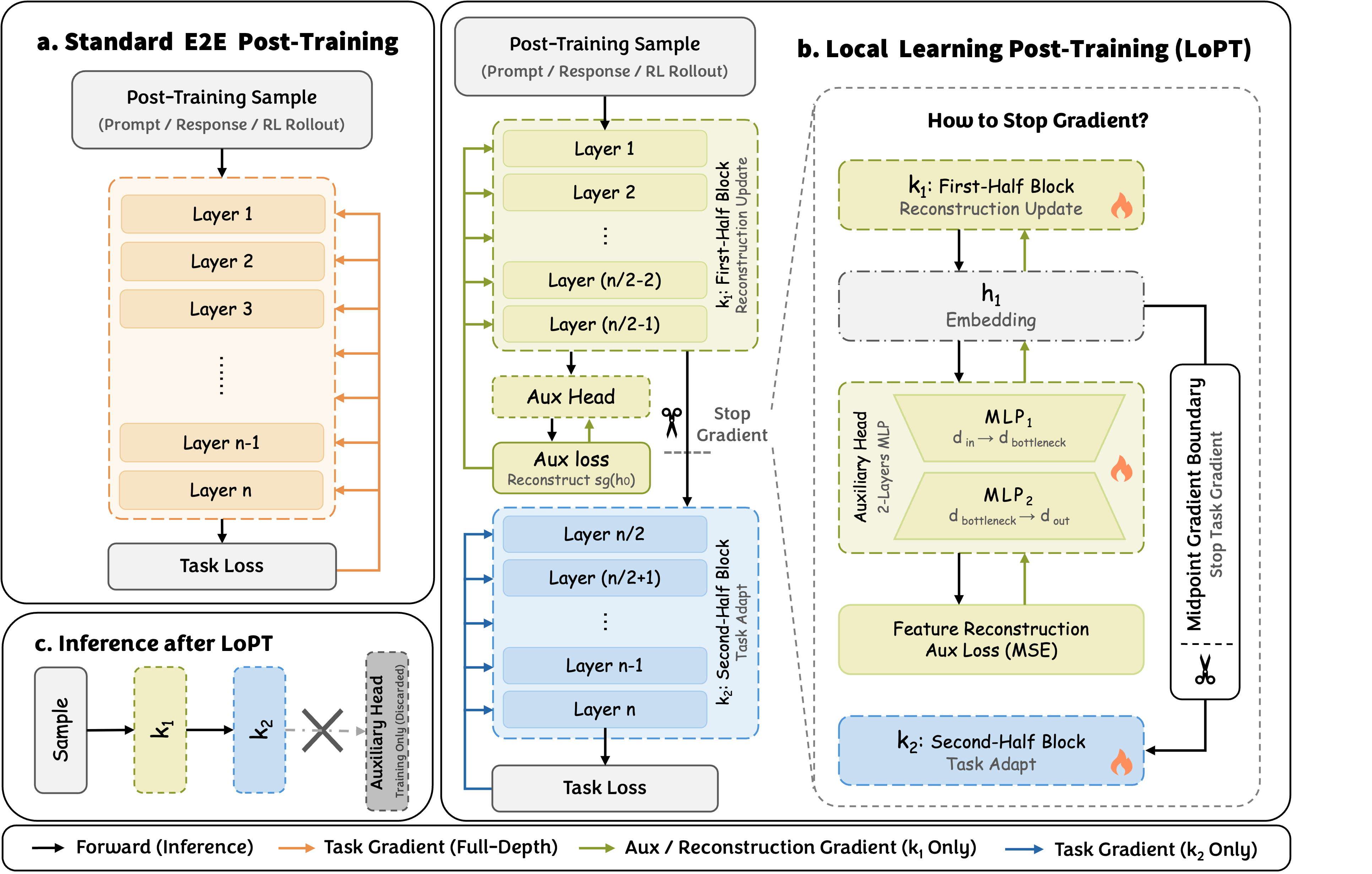}
    \caption{\textbf{Overview of Local-Learning Post-Training (LoPT).}
    (a) Standard E2E post-training propagates task gradients through the full transformer stack.
    (b) LoPT inserts a midpoint stop-gradient boundary: \(k_2\) is updated by the task loss, while \(k_1\) remains trainable through a local feature-reconstruction objective implemented with a lightweight auxiliary head.
    (c) The auxiliary head is used only during training and discarded at inference, so inference follows the standard \(k_1 \rightarrow k_2\) forward path without extra computation.}
    \label{fig:lopt_overview}
\end{figure}

\subsection{Local-Learning View of Gradient Reach}
\label{sec:method_local_learning}

In contrast to standard end-to-end (E2E) post-training, LoPT adopts a local-learning perspective: instead of allowing the task objective to directly supervise every layer, it restricts task-level supervision to the part of the model primarily responsible for task adaptation. As illustrated in Fig.~\ref{fig:lopt_overview}, the first-half block still participates in the same forward computation, but is no longer updated directly by the task loss.

Concretely, LoPT inserts a stop-gradient boundary between $k_1$ and $k_2$ by:
\begin{equation}
    \mathbf{h}_1=f_{k_1}(\mathbf{h}_0;\theta_1),
    \qquad
    \hat{\mathbf{h}}_1=\mathrm{sg}(\mathbf{h}_1),
    \qquad
    \mathbf{y}=f_{k_2}(\hat{\mathbf{h}}_1;\theta_2),
\end{equation}
where $\mathrm{sg}(\cdot)$ denotes the stop-gradient operator, acting as the identity in the forward pass while setting the backward Jacobian to zero. 

The task loss therefore remains a function of the representation produced by the first-half block, while its gradient is prevented from reaching the first-half-block parameters:
\begin{equation}
\begin{aligned}
    \nabla_{\theta_1}\mathcal{L}_{\mathrm{task}}
    &=
    \frac{\partial \mathcal{L}_{\mathrm{task}}}{\partial \mathbf{y}}
    \frac{\partial \mathbf{y}}{\partial \hat{\mathbf{h}}_1}
    \underbrace{\frac{\partial \hat{\mathbf{h}}_1}{\partial \mathbf{h}_1}}_{=\mathbf{0}}
    \frac{\partial \mathbf{h}_1}{\partial \theta_1}
    =
    \mathbf{0},
    \\
    \nabla_{\theta_2}\mathcal{L}_{\mathrm{task}}
    &=
    \frac{\partial \mathcal{L}_{\mathrm{task}}}{\partial \mathbf{y}}
    \frac{\partial \mathbf{y}}{\partial \theta_2}.
\end{aligned}
\label{eq:lopt_task_grad}
\end{equation}
Thus, the task-induced backward path terminates at the boundary. Since the boundary alone would make the first-half block passive with respect to training, LoPT assigns it a local auxiliary objective:
\begin{equation}
    \mathcal{L}_{\mathrm{local}}
    =
    \mathcal{L}_{\mathrm{local}}(\mathbf{h}_0,\mathbf{h}_1;\phi),
\end{equation}
where $\phi$ denotes auxiliary parameters. $\mathcal{L}_{\mathrm{local}}$ is not meant to solve the downstream task independently; instead, it provides a representation-preserving signal that keeps the first-half block plastic while maintaining an interface compatible with the adapting second-half block. The parameter update is:
\begin{equation}
\begin{aligned}
    (\theta_1,\phi)
    &\leftarrow
    (\theta_1,\phi)
    -
    \eta_1
    \nabla_{\theta_1,\phi}
    \lambda_{\mathrm{local}}\mathcal{L}_{\mathrm{local}},
    \\
    \theta_2
    &\leftarrow
    \theta_2
    -
    \eta_2
    \nabla_{\theta_2}\mathcal{L}_{\mathrm{task}}.
\end{aligned}
\label{eq:lopt_update}
\end{equation}
By separating the two update paths, LoPT optimizes the second-half block for the post-training objective while keeping the first-half block trainable through representation maintenance. In contrast to E2E post-training, task gradients no longer propagate through the full model depth.

\subsection{Local Learning for LLM Post-Training}
\label{sec:method_lopt}

LoPT controls the reach of task supervision by changing the backward path while leaving the post-training objective unchanged. To define the gradient boundary, we use the same partition as in Eq.~\ref{eq:model_partition} and set
\begin{equation}
    s=\left\lfloor \frac{N}{2} \right\rfloor .
\end{equation}
After this choice, $k_1$ and $k_2$ correspond to the first-half and second-half blocks, respectively. The midpoint boundary is used as a simple architecture-agnostic default, avoiding split-location tuning while making gradient reach comparable across model families. For architectures with tied input embeddings and output heads, we treat the two modules as distinct parameter tensors when defining the split, preventing output-head gradients from bypassing the LoPT boundary and updating the input embedding.

\paragraph{Training objectives of LoPT.}
LoPT keeps the standard supervised post-training objective unchanged in form, but changes how its gradients are routed through the model. The task branch receives the boundary representation produced by the first-half block while stopping the task gradient at the boundary. Formally, we compute
\begin{equation}
    \mathbf{h}_1=f_{k_1}(\mathbf{h}_0),
    \qquad
    \hat{\mathbf{h}}_1=\mathrm{sg}(\mathbf{h}_1).
\end{equation}
For supervised fine-tuning, we use the standard causal language-modeling loss over the non-padding prediction positions in the formatted instruction-output sequence:
\begin{equation}
\mathcal{L}_{\mathrm{SFT}}
=
-\frac{1}{|\mathcal{T}|}
\sum_{t\in\mathcal{T}}
\log p_{\theta_2}
\left(
x_{t+1}\mid x_{\leq t};\, \hat{\mathbf{h}}_1
\right),
\label{eq:lopt_sft}
\end{equation}
where $\mathcal{T}$ denotes the set of non-padding prediction positions used for the SFT loss. The forward prediction still depends on both halves of the model because $\hat{\mathbf{h}_1}$ is produced by $k_1$, but for the SFT objective it is treated as a constant in backpropagation:
\begin{equation}
    \nabla_{\theta_1}\mathcal{L}_{\mathrm{SFT}}=\mathbf{0},
    \qquad
    \nabla_{\theta_2}\mathcal{L}_{\mathrm{SFT}}\not\equiv\mathbf{0}.
\end{equation}

For GRPO~\citep{shao2024deepseekmath}, given a prompt $\mathbf{x}$, sampled responses $\{o_i\}_{i=1}^{G}$, and normalized advantages $\hat{A}_i$, the current policy log-probability is computed through the same stop-gradient boundary:
\begin{equation}
\begin{aligned}
\mathbf{h}_{1,i,t}
&=
f_{k_1}(\mathbf{h}_{0,i,t};\theta_1), \\
\hat{\mathbf{h}_{1,i,t}}
&=
\mathrm{sg}(\mathbf{h}_{1,i,t}), \\
\ell^{\mathrm{LoPT}}_{i,t}
&=
\log \pi_{\theta_2}
\left(
o_{i,t}\mid \mathbf{x},o_{i,<t};\,
\hat{\mathbf{h}_{1,i,t}}
\right).
\end{aligned}
\end{equation}
The token-level probability ratio and clipped GRPO loss are computed as:
\begin{equation}
\begin{aligned}
r_{i,t}
&=
\exp\left(
\ell^{\mathrm{LoPT}}_{i,t}
-
\log \pi_{\mathrm{old}}
\left(
o_{i,t}\mid \mathbf{x},o_{i,<t}
\right)
\right), \\
\mathcal{L}_{\mathrm{GRPO}}
&=
-\frac{1}{G}
\sum_{i=1}^{G}
\frac{1}{|o_i|}
\sum_{t=1}^{|o_i|}
\min \left(
r_{i,t}\hat{A}_i,\,
\mathrm{clip}(r_{i,t},1-\epsilon,1+\epsilon)\hat{A}_i
\right).
\end{aligned}
\label{eq:lopt_grpo}
\end{equation}
Although the current policy is evaluated using the boundary representation produced by $k_1$, the stop-gradient boundary prevents the resulting policy-gradient signal from entering the first-half block:
\begin{equation}
    \nabla_{\theta_1}\mathcal{L}_{\mathrm{GRPO}}=\mathbf{0},
    \qquad
    \nabla_{\theta_2}\mathcal{L}_{\mathrm{GRPO}}\not\equiv\mathbf{0}.
\end{equation}

\paragraph{Feature reconstruction for preserving useful lower-level representations.}
The stop-gradient boundary removes direct task supervision from the first-half block. To keep this block trainable, LoPT updates it with a local feature-reconstruction objective instead of a task-level decoding loss. We attach a lightweight bottleneck MLP to the boundary representation:
\begin{equation}
g_\phi(\mathbf{h})
=
W_2\,\mathrm{GELU}\!\left(W_1\,\mathrm{LayerNorm}(\mathbf{h})+\mathbf{b}_1\right)
+\mathbf{b}_2,
\end{equation}
where $W_1\in\mathbb{R}^{\frac{d}{4}\times d}$ and $W_2\in\mathbb{R}^{d\times \frac{d}{4}}$. 
The auxiliary objective then can be written as:
\begin{equation}
\mathcal{L}_{\mathrm{aux}}
=
\frac{1}{BLD}
\sum_{b=1}^{B}
\sum_{\ell=1}^{L}
\left\|
g_\phi(\mathbf{h}_{1,b,\ell})
-
\mathrm{sg}(\mathbf{h}_{0,b,\ell})
\right\|_2^2 .
\label{eq:lopt_aux}
\end{equation}

$\mathcal{L}_{\mathrm{aux}}$ reconstructs the stop-gradient input embedding state from the first-half output and is intended to preserve useful lower-level representations and maintain a compatible interface for the adapting second-half block, rather than to solve the downstream task independently. The auxiliary update optimizes both the first-half block and the auxiliary head \(g_\phi\), whereas the task update optimizes only the second-half block:
\begin{equation}
\begin{aligned}
    (\theta_1,\phi)
    &\leftarrow
    (\theta_1,\phi)
    -
    \eta_1
    \nabla_{\theta_1,\phi}
    \lambda_{\mathrm{aux}}\mathcal{L}_{\mathrm{aux}}, \\
    \theta_2
    &\leftarrow
    \theta_2
    -
    \eta_2
    \nabla_{\theta_2}\mathcal{L}_{\mathrm{task}},
    \qquad
    \mathcal{L}_{\mathrm{task}}\in
    \{\mathcal{L}_{\mathrm{SFT}},\mathcal{L}_{\mathrm{GRPO}}\}.
\end{aligned}
\label{eq:lopt_updates}
\end{equation}

Appendix~\ref{app:local_objective_ablation} further shows replacing feature reconstruction with a local next-token-prediction objective severely degrades performance, supporting our choice of a feature reconstruction objective.

\paragraph{Implemented update order.}
Each LoPT step uses two separated backward paths. We first compute the auxiliary reconstruction loss and update \((\theta_1,\phi)\). We then recompute the boundary activation with the updated first-half block, detach it with \(\mathrm{sg}(\cdot)\), and update only \(\theta_2\) using the task loss. This order avoids using a stale boundary activation and allows the first-half computation graph from the auxiliary pass to be released before constructing the task-loss graph. The auxiliary head is used only during training and discarded at inference time; inference follows the standard forward path through \(k_1\) and \(k_2\) without additional computation. Additional implementation details are provided in Appendices~\ref{app:exact_impl} and~\ref{app:aux_net}.

\section{Experiments}
\label{sec:setup}


\subsection{Implementations}
\subsubsection{Models and Datasets}
We train LoPT on Qwen3-4B~\citep{yang2025qwen3}, Qwen2.5-7B-Instruct~\citep{yang2024qwen2}, and Llama-3.1-8B-Instruct~\citep{grattafiori2024llama3}. For supervised fine-tuning, we use Alpaca-52K~\citep{taori2023alpaca}, Tulu-3 SFT Mix 100K~\citep{lambert2024tulu3}, Magpie-Pro 100K~\citep{xu2025magpie}, and MetaMathQA 100K~\citep{yu2024metamathqa}, covering both general instruction tuning and math-focused supervised adaptation. For GRPO, we use GSM8K-train~\citep{cobbe2021gsm8k} and NuminaMath~\citep{li2024numinamath} as training datasets.

\paragraph{Experimental protocol.}

We evaluate LoPT under both supervised fine-tuning (SFT) and GRPO. The main-text tables report representative SFT results on Alpaca-52K and MetaMathQA 100K and GSM8K GRPO results, while Appendix~\ref{app:full_sft} and Appendix~\ref{app:full_grpo} provide the full SFT sweep over Alpaca-52K, Tulu-3 SFT Mix 100K, Magpie-Pro 100K, and MetaMathQA 100K and the full GRPO sweep over GSM8K and NuminaMath. Unless otherwise stated, trainable configurations use three random seeds, and all deltas, win counts, and seven-benchmark averages are computed from seed means. Models are evaluated with lm-eval-harness on seven benchmarks covering instruction following, reasoning, factuality, commonsense, and language understanding. Appendix~\ref{app:training_details} and Appendix~\ref{app:evaluation_details} provide hyperparameters, decoding settings, seed averaging, and evaluation details.

\subsection{Main Results}
\label{sec:results}

\subsubsection{Quality--Efficiency Trade-off}
\label{sec:capability}

Tables~\ref{tab:sft_quality_efficiency} and~\ref{tab:grpo_quality_efficiency} compare LoPT with standard end-to-end post-training in both quality and training cost. Base denotes the pretrained checkpoint, and E2E denotes full-depth task-gradient propagation. For SFT, we report matched runs on Alpaca-52K and MetaMathQA 100K; for GRPO, we train on GSM8K-train and report the cost of the policy-update phase, since rollout generation is unchanged between E2E and LoPT.

\paragraph{SFT.}

LoPT is most beneficial when the SFT data are weakly aligned with held-out evaluation. On Alpaca-52K, E2E causes large drops on several non-rehearsed benchmarks, especially for Llama-3.1-8B-Instruct: GSM8K falls from 69.67 to 46.62 and IFEval from 44.54 to 30.50, while LoPT recovers them to 72.10 and 46.03. The same pattern appears more mildly on Qwen2.5-7B-Instruct, and the average LoPT-over-E2E gains on Alpaca-52K are $+1.32$, $+1.97$, and $+7.22$ across the three model blocks. On MetaMathQA 100K, where the training data are closer to GSM8K-style reasoning, LoPT is mostly near parity with E2E, but still improves the Llama-3.1-8B-Instruct average by $+2.77$. Across matched SFT runs, LoPT reduces peak memory by 24--36\% and increases throughput by 2--7\%, showing that the quality gains do not come from additional training cost.

\paragraph{GRPO.}
Under GRPO, LoPT mainly provides lower-cost policy updates while maintaining near-parity quality. Since GSM8K is both the training task and the main optimized metric, LoPT differs only slightly from E2E on GSM8K, with gains below $0.4\pp$ across the three models. The held-out columns still show modest preservation effects, such as Qwen2.5-7B-Instruct improving IFEval from 56.01 to 59.70 and keeping TruthfulQA closer to Base. The efficiency gains are clearer: LoPT reduces policy-update peak memory by 21--24\% and reduces update step time by 7--21\%. These savings follow from the same mechanism as in SFT: the task-induced backward graph stops at the midpoint boundary, reducing activation storage and backward computation below it. Appendix~\ref{app:extended_efficiency} further shows that this gradient-routing gain is compatible with gradient checkpointing, LoRA, DeepSpeed-ZeRO-1/2/3, and pipeline parallelism.

\begin{table}[htbp]
\caption{SFT evaluation results and training cost on lm-eval-harness. Each trainable configuration is run with three different random seeds, and reported values are means over the three runs. Each model is trained on Alpaca-52K or MetaMathQA and evaluated under Base, E2E, and LoPT. Peak memory is measured in GB, and speed is measured as total training throughput in tokens/s.}
\centering
\scriptsize
\setlength{\tabcolsep}{2.8pt}
\renewcommand{\arraystretch}{1.08}
\resizebox{\textwidth}{!}{
\begin{tabular}{lllccccccccc}
\toprule
Dataset & Model & Method
& MMLU$\uparrow$ & IFEval$\uparrow$ & ARC-C$\uparrow$ & GSM8K$\uparrow$
& HellaSwag$\uparrow$ & TruthfulQA$\uparrow$ & Winogrande$\uparrow$
& Peak Mem (GB)$\downarrow$ & Speed (tok/s)$\uparrow$ \\
\midrule
\multirow{9}{*}{\textbf{Alpaca-52K}}
& \multirow{3}{*}{Qwen3-4B-Base}
& Base
& 70.07 & 23.84 & 57.84 & \textbf{85.14} & 65.62 & \textbf{54.85} & 67.56
& - & - \\
& & E2E
& \textbf{70.12} & 25.87 & 61.77 & 79.45 & 69.80 & 53.07 & \textbf{68.19}
& 31.5 & 21{,}020 \\
& & \textbf{LoPT (ours)}
& 69.99 & \textbf{29.39} & \textbf{61.95} & 84.62 & \textbf{69.97} & 53.91 & 67.69
& 21.9~($\downarrow30\%$) & 21{,}469~($\uparrow$2\%) \\
\cmidrule(lr){2-12}

& \multirow{3}{*}{Qwen2.5-7B-Ins}
& Base
& \textbf{74.17} & \textbf{57.30} & \textbf{66.72} & \textbf{83.47} & 79.35 & \textbf{64.75} & 74.51
& - & - \\
& & E2E
& 73.09 & 40.30 & 60.58 & 74.22 & 78.55 & 52.28 & \textbf{75.77}
& 57.7 & 15{,}561 \\
& & \textbf{LoPT (ours)}
& 74.08 & 41.22 & 63.96 & 79.38 & \textbf{79.46} & 54.86 & 75.63
& 37.2~($\downarrow$36\%) & 15{,}971~($\uparrow$3\%) \\
\cmidrule(lr){2-12}

& \multirow{3}{*}{Llama-3.1-8B-Ins}
& Base
& 68.20 & 44.54 & \textbf{60.49} & 69.67 & \textbf{79.42} & \textbf{54.56} & \textbf{77.50}
& - & - \\
& & E2E
& \textbf{68.67} & 30.50 & 57.59 & 46.62 & 78.31 & 47.39 & 76.09
& 60.6 & 14{,}483 \\
& & \textbf{LoPT (ours)}
& 68.48 & \textbf{46.03} & 60.41 & \textbf{72.10} & 79.38 & 52.67 & 76.64
& 38.9~($\downarrow$36\%) & 14{,}882~($\uparrow$3\%) \\

\midrule
\multirow{9}{*}{\textbf{MetaMathQA 100K}}
& \multirow{3}{*}{Qwen3-4B-Base}
& Base
& 70.07 & \textbf{23.84} & 57.84 & \textbf{85.14} & 65.62 & \textbf{54.85} & 67.56
& - & - \\
& & E2E
& \textbf{70.18} & 20.89 & 63.57 & 81.20 & 72.43 & 54.49 & 68.69
& 33.0 & 25{,}019 \\
& & \textbf{LoPT (ours)}
& 70.17 & 22.26 & \textbf{63.64} & 81.21 & \textbf{72.53} & 54.55 & \textbf{68.88}
& 25.2~($\downarrow$24\%) & 26{,}670~($\uparrow$7\%) \\
\cmidrule(lr){2-12}

& \multirow{3}{*}{Qwen2.5-7B-Ins}
& Base
& 74.17 & \textbf{57.30} & \textbf{66.72} & 83.47 & 79.35 & \textbf{64.75} & \textbf{74.51}
& - & - \\
& & E2E
& 74.13 & 51.76 & 63.91 & 86.80 & 80.85 & 62.63 & 71.51
& 58.7 & 19{,}567 \\
& & \textbf{LoPT (ours)}
& \textbf{74.23} & 51.46 & 64.33 & \textbf{87.21} & \textbf{80.97} & 62.66 & 74.03
& 38.4~($\downarrow$35\%) & 20{,}511~($\uparrow$5\%) \\
\cmidrule(lr){2-12}

& \multirow{3}{*}{Llama-3.1-8B-Ins}
& Base
& \textbf{68.20} & \textbf{44.54} & 60.49 & 69.67 & 79.42 & \textbf{54.56} & \textbf{77.50}
& - & - \\
& & E2E
& 65.29 & 34.94 & 58.70 & 76.79 & 77.70 & 52.06 & 76.56
& 61.4 & 18{,}028 \\
& & \textbf{LoPT (ours)}
& 67.94 & 43.99 & \textbf{60.92} & \textbf{79.27} & \textbf{79.71} & 52.93 & 76.64
& 39.9~($\downarrow$35\%) & 19{,}194~($\uparrow$6\%) \\
\bottomrule
\end{tabular}
}
\vspace{0.5mm}
\label{tab:sft_quality_efficiency}
\end{table}

\begin{table}[htbp]
\caption{GRPO evaluation results and policy-update cost on lm-eval-harness. Trainable rows are three-seed means. Models are trained on GSM8K and evaluated under Base, E2E, and LoPT. Peak memory is in GB; step time is policy-update time in s/step.}
\centering
\label{tab:grpo_quality_efficiency}
\scriptsize
\setlength{\tabcolsep}{2.8pt}
\renewcommand{\arraystretch}{1.08}
\resizebox{\textwidth}{!}{
\begin{tabular}{lllccccccccc}
\toprule
Dataset & Model & Method
& MMLU$\uparrow$ & IFEval$\uparrow$ & ARC-C$\uparrow$ & GSM8K$\uparrow$
& HellaSwag$\uparrow$ & TruthfulQA$\uparrow$ & Winogrande$\uparrow$
& Peak Mem (GB)$\downarrow$ & Step Time (s/step)$\downarrow$ \\
\midrule
\multirow{9}{*}{\textbf{GSM8K}}
& \multirow{3}{*}{Qwen3-4B-Base}
& Base
& 70.07 & 23.84 & 57.84 & 85.14 & 65.62 & \textbf{54.85} & \textbf{67.56}
& - & - \\
& & E2E
& 70.01 & \textbf{24.03} & 62.03 & 88.35 & 68.76 & 54.78 & 67.17
& 42.55 & 1.42 \\
& & \textbf{LoPT (ours)}
& \textbf{70.10} & 23.84 & \textbf{62.29} & \textbf{88.63} & \textbf{69.95} & 54.62 & 66.85
& 32.89~($\downarrow$23\%) & 1.32~($\downarrow$7\%) \\
\cmidrule(lr){2-12}

& \multirow{3}{*}{Qwen2.5-7B-Ins}
& Base
& 74.17 & 57.30 & \textbf{66.72} & 83.47 & 79.35 & 64.75 & 74.51
& - & - \\
& & E2E
& \textbf{74.22} & 56.01 & \textbf{66.72} & 88.70 & 81.23 & 64.52 & 75.22
& 75.39 & 1.71 \\
& & \textbf{LoPT (ours)}
& 74.18 & \textbf{59.70} & 66.38 & \textbf{89.09} & \textbf{81.40} & \textbf{64.83} & \textbf{75.40}
& 59.66~($\downarrow$21\%) & 1.56~($\downarrow$9\%) \\
\cmidrule(lr){2-12}

& \multirow{3}{*}{Llama-3.1-8B-Ins}
& Base
& 68.20 & 44.54 & 60.49 & 69.67 & 79.42 & \textbf{54.56} & \textbf{77.50}
& - & - \\
& & E2E
& 68.44 & \textbf{47.69} & \textbf{61.69} & 82.02 & 80.39 & 54.53 & 76.80
& 62.62 & 2.10 \\
& & \textbf{LoPT (ours)}
& \textbf{68.55} & 47.03 & 61.43 & \textbf{82.40} & \textbf{80.47} & 54.11 & 77.35
& 47.54~($\downarrow$24\%) & 1.65~($\downarrow$21\%) \\
\bottomrule
\end{tabular}
}
\vspace{0.5mm}
\end{table}

\subsubsection{Localized Parameter Drift}
\label{sec:param_drift}

LoPT changes where task-induced gradients act; it does not freeze the first-half block or remove it from the forward computation. The task objective updates \(k_2\), while \(k_1\) remains trainable through feature reconstruction. We therefore expect persistent post-training movement to concentrate above the midpoint boundary, with \(k_1\) staying close to the base checkpoint.

We measure layer-wise relative parameter drift between the tuned and base checkpoints:
\begin{equation}
    \Delta_\ell
    =
    \frac{
    \left\|
    \theta^{\mathrm{tuned}}_\ell
    -
    \theta^{\mathrm{base}}_\ell
    \right\|_2
    }{
    \left\|
    \theta^{\mathrm{base}}_\ell
    \right\|_2
    } .
\end{equation}
Figure~\ref{fig:param_drift} shows that E2E spreads drift across both \(k_1\) and \(k_2\), whereas LoPT keeps \(k_1\) close to the base checkpoint and concentrates most persistent movement in \(k_2\). This pattern also appears when drift is decomposed into attention and MLP parameters, supporting the intended gradient-routing mechanism: the task branch adapts the second-half block, while the first-half block is updated by a local objective rather than by the task loss.

\begin{figure}[htbp]
    \centering
    \includegraphics[width=0.9\linewidth]{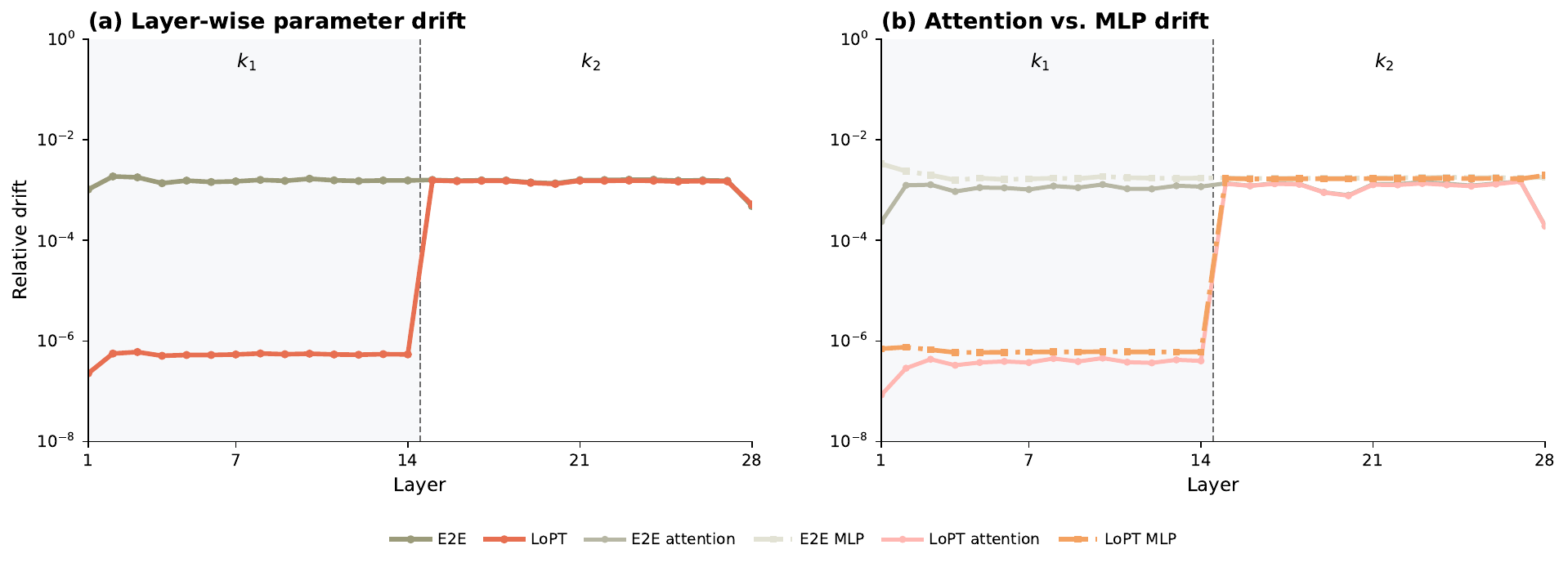}
    \caption{\textbf{Layer-wise parameter drift under GRPO.}
    Drift is measured as the relative parameter change from the base checkpoint to the final tuned checkpoint and averaged over three independent GRPO runs with different random seeds.
    LoPT concentrates persistent movement in \(k_2\), whereas E2E spreads movement across the full model.}
    \label{fig:param_drift}
\end{figure}

Small first-half drift should not be mistaken for functional inactivity. Table~\ref{tab:front14_perturb} calibrates this scale by perturbing only the first 14 layers of the Qwen2.5-7B-Instruct base checkpoint with a relative drift of \(2\times10^{-7}\). The average score changes from 71.47 to 70.98, a drop of \(0.49\pp\), with mostly small negative shifts. This diagnostic does not suggest that random perturbations are useful; it only shows that very small early-layer changes can still be visible at the benchmark level. Thus, LoPT's small \(k_1\) drift is better interpreted as controlled auxiliary-driven movement around the pretrained interface, not as freezing.

\begin{table}[htbp]
\caption[Front-layer random perturbation sensitivity check]{\textbf{Front-layer random perturbation sensitivity check.} Qwen2.5-7B-Ins base checkpoint; values are three-perturbation-seed means (\%).}
\label{tab:front14_perturb}
\centering
\scriptsize
\setlength{\tabcolsep}{5pt}
\renewcommand{\arraystretch}{1.08}
\begin{tabular}{lccc}
\toprule
Benchmark & Base & Front14 \( \mathrm{rel\_drift}=2\times10^{-7} \) perturbation & \(\Delta\) \\
\midrule
MMLU & \textbf{74.17} & 73.26 & -0.91 \\
IFEval & 57.30 & \textbf{58.59} & +1.29 \\
ARC-C & \textbf{66.72} & 66.15 & -0.57 \\
GSM8K & \textbf{83.47} & 82.01 & -1.46 \\
HellaSwag & \textbf{79.35} & 79.16 & -0.19 \\
TruthfulQA & \textbf{64.75} & 64.02 & -0.73 \\
Winogrande & \textbf{74.51} & 73.66 & -0.85 \\
\midrule
Avg. & \textbf{71.47} & 70.98 & -0.49 \\
\bottomrule
\end{tabular}
\vspace{0.45em}
\begin{flushleft}
\footnotesize{
The perturbation is applied only to the first 14 layers of the base checkpoint; the remaining layers are unchanged. This diagnostic is not a post-training method and is not used as evidence of quality improvement. It only calibrates that small early-layer perturbations can yield measurable but near-parity benchmark shifts.
}
\end{flushleft}
\end{table}

\subsection{Ablation Study}

\subsubsection{Reconstruction-Guided Updates versus Freezing the First-Half Block.}
We further test whether the reconstruction objective contributes beyond LoPT. We compare LoPT with freeze-\(k_1\), a matched control that keeps the same settings as LoPT, but removes the first-half-block reconstruction loss and update. As shown in Table~\ref{tab:freeze_grpo}, LoPT outperforms freeze-\(k_1\) on all seven benchmarks in lm-eval-harness, with an average gain of \(+1.24\pp\) under SFT. Under GRPO, it improves six of seven benchmarks, with MMLU nearly unchanged at \(-0.02\pp\), and yields an average gain of \(+1.64\pp\). These results also suggest that the benefit is not due to arbitrary same-scale parameter movement. Random \(2\times10^{-7}\)-level perturbations in the first 14 layers mostly degrade the base model, whereas LoPT's reconstruction-guided updates improve over the frozen-\(k_1\) control. Therefore, the useful effect of LoPT comes from keeping the lower block trainable through a local, feature-level objective while preventing downstream task gradients from directly reshaping it.

\begin{table}[htbp]
\caption[freeze-\(k_1\) ablations in SFT and GRPO]{\textbf{freeze-\(k_1\) ablations in SFT and GRPO.} Values are three-seed means (\%). \(\Delta\) denotes LoPT (ours) minus freeze-\(k_1\).}
\label{tab:freeze_grpo}
\centering
\scriptsize
\setlength{\tabcolsep}{2.8pt}
\renewcommand{\arraystretch}{1.08}
\resizebox{0.94\textwidth}{!}{
\begin{tabular}{lccccc|ccccc}
\toprule
& \multicolumn{5}{c}{\textbf{SFT: MetaMathQA 100K / Qwen2.5-7B-Ins}}
& \multicolumn{5}{c}{\textbf{GRPO: GSM8K / Qwen2.5-7B-Ins}} \\
\cmidrule(lr){2-6}\cmidrule(lr){7-11}
Benchmark
& Base & E2E & freeze-\(k_1\) & LoPT (ours) & \(\Delta\)
& Base & E2E & freeze-\(k_1\) & LoPT (ours) & \(\Delta\) \\
\midrule
MMLU
& 74.17 & 74.13 & 74.06 & \textbf{74.23} & \qnear{+0.17}
& 74.17 & 74.22 & \textbf{74.20} & 74.18 & \qnear{-0.02} \\

IFEval
& 57.30 & 51.76 & 50.46 & \textbf{51.46} & \qpos{+1.00}
& 57.30 & 56.01 & 57.96 & \textbf{59.70} & \qpos{+1.74} \\

ARC-C
& 66.72 & 63.91 & 63.76 & \textbf{64.33} & \qnear{+0.57}
& 66.72 & 66.72 & 64.55 & \textbf{66.38} & \qpos{+1.83} \\

GSM8K
& 83.47 & 86.80 & 85.91 & \textbf{87.21} & \qpos{+1.30}
& 83.47 & 88.70 & 87.56 & \textbf{89.09} & \qpos{+1.53} \\

HellaSwag
& 79.35 & 80.85 & 78.84 & \textbf{80.97} & \qpos{+2.13}
& 79.35 & 81.23 & 79.37 & \textbf{81.40} & \qpos{+2.03} \\

TruthfulQA
& 64.75 & 62.63 & 60.61 & \textbf{62.66} & \qpos{+2.05}
& 64.75 & 64.52 & 62.72 & \textbf{64.83} & \qpos{+2.11} \\

Winogrande
& 74.51 & 71.51 & 72.56 & \textbf{74.03} & \qpos{+1.47}
& 74.51 & 75.22 & 73.16 & \textbf{75.40} & \qpos{+2.24} \\
\midrule
Avg. \(\Delta\)
& \multicolumn{4}{c}{} & \qpos{+1.24}
& \multicolumn{4}{c}{} & \qpos{+1.64} \\
\bottomrule
\end{tabular}
}
\vspace{0.45em}
\begin{flushleft}
\footnotesize{
For each benchmark, bolding compares freeze-\(k_1\) and LoPT (ours); Base and E2E are included as context. Green marks \(\Delta \ge +1\pp\), and gray marks differences within \(\pm1\pp\). These are diagnostic ablations rather than a broad freeze-\(k_1\) sweep.
}
\end{flushleft}
\end{table}

\subsubsection{Ablation on the Number of Partitions}
We ablate the number of LoPT partitions in Table~\ref{tab:k4_ablation}, comparing the default two-block design (\(k=2\), one midpoint boundary)
with a more heavily partitioned variant (\(k=4\)). This targeted ablation tests whether further shortening task-gradient reach is beneficial.
\begin{table}[htbp]
\caption{Midpoint-boundary \(k\) ablation on Qwen2.5-7B-Instruct; values are three-seed means.}
\centering
\scriptsize
\setlength{\tabcolsep}{7pt}
\renewcommand{\arraystretch}{1.08}
\begin{tabular}{lccc}
\toprule
Benchmark & E2E & LoPT \(k=2\) & LoPT \(k=4\) \\
\midrule
MMLU & 74.13 & \textbf{74.23} & 74.09 \\
IFEval & 51.76 & 51.46 & \textbf{52.60} \\
ARC-C & 63.91 & \textbf{64.33} & 63.18 \\
GSM8K & 86.80 & \textbf{87.21} & 85.59 \\
HellaSwag & 80.85 & \textbf{80.97} & 80.30 \\
TruthfulQA & 62.63 & \textbf{62.66} & 62.49 \\
Winogrande & 71.51 & \textbf{74.03} & 73.09 \\
\bottomrule
\end{tabular}
\vspace{0.5em}

\label{tab:k4_ablation}
\end{table}

The results show that more partitions are not consistently better. The \(k=4\) variant improves IFEval, suggesting that a shorter task-gradient reach may better preserve certain held-out
instruction-following behavior. However, it underperforms \(k=2\) on most benchmarks, especially the training-aligned GSM8K evaluation. This indicates that over-partitioning can make
the task-gradient reach too narrow, limiting effective task-driven adaptation. We therefore use the single-midpoint \(k=2\) setting as the
default, which better balances capability preservation and task adaptation.

\section{Conclusion}

In this paper, we revisited the default assumption of end-to-end gradient propagation in LLM post-training and argued that gradient reach should be treated as an explicit design variable rather than a fixed implementation choice. We introduced \textbf{Local-Learning Post-Training (LoPT)}, a simple scheme that inserts a single midpoint gradient boundary, optimizes the second-half block with the task objective, and keeps the first-half block trainable through a lightweight feature-reconstruction objective. Across both SFT and GRPO settings, LoPT delivers competitive or improved downstream performance while reducing training memory, improving efficiency, and better preserving pretrained capabilities. Because this form of gradient isolation is orthogonal to common parameter-efficient and systems optimizations, it also composes naturally with existing post-training stacks. More broadly, our results suggest that effective post-training does not necessarily require full-depth backward coupling, and we hope this perspective opens a broader design space for methods that jointly optimize capability, efficiency, and representation preservation.




\newpage
\appendix

\definecolor{llwinbg}{RGB}{226,245,232}
\definecolor{lllossbg}{RGB}{252,228,228}
\definecolor{llneutralbg}{RGB}{242,242,242}
\providecommand{\llwin}[1]{\cellcolor{llwinbg}{#1}}
\providecommand{\llloss}[1]{\cellcolor{lllossbg}{#1}}
\providecommand{\llneutral}[1]{\cellcolor{llneutralbg}{#1}}
\providecommand{\pp}{\mathrm{pp}}

\section{Appendix Roadmap}
\label{app:roadmap}

The appendix is organized as follows. Appendix~\ref{sec:limitations} discusses the limitations of this empirical study. Appendix~\ref{sec:related} discusses related work on local learning and memory-efficient LLM post-training. Appendix~\ref{app:experimental_details} provides the shared training and evaluation protocol. Appendix~\ref{app:full_sft} gives the complete SFT sweep across four datasets and three model families, complementing the representative main-text SFT table (Table~\ref{tab:sft_quality_efficiency}). Appendix~\ref{app:full_grpo} gives the full GSM8K and NuminaMath GRPO results behind the main-text GRPO slice (Table~\ref{tab:grpo_quality_efficiency}). Appendix~\ref{app:local_objective_ablation} further ablates the lower-block local objective, comparing feature reconstruction with local next-token-prediction objectives. Appendix~\ref{app:extended_efficiency} expands the efficiency analysis with matched-stack profiling, pipeline-parallel layout checks, and single-/dual-GPU scaling results. Appendix~\ref{app:lora_results} reports LoRA compatibility results, and Appendix~\ref{app:scale32b} provides a limited 32B GRPO sanity check. Appendices~\ref{app:grpo_details} and~\ref{app:exact_impl} document the GRPO hyperparameters and the exact LoPT gradient-routing structure. Appendices~\ref{app:metrics}--\ref{app:aux_net} summarize metric conventions, compute budget, and auxiliary-head details. Appendix~\ref{app:theory} provides a theoretical analysis of LoPT, including task-gradient reach, first-half drift, feature reconstruction, held-out retention, convergence of the separated update paths, and memory implications.

\section{Limitations}
\label{sec:limitations}

This work has two main limitations. First, due to compute constraints, our main quality sweep focuses on 4B--8B models, with an additional 32B GRPO sanity check in Appendix~\ref{app:scale32b}. This leaves open the question of whether the same trends hold for frontier-scale open-weight models, including recent MoE models with hundreds of billions of total parameters. Second, our default design uses a simple midpoint split, and Table~\ref{tab:k4_ablation} only compares it with one more fine-grained \(k{=}4\) variant. A broader search over split granularity and boundary placement could provide a more precise understanding of when and where local post-training is most effective.


\section{Additional Experimental Details}
\label{app:experimental_details}

\subsection{Training Details}
\label{app:training_details}
Unless otherwise specified, all main SFT and GRPO quality/efficiency experiments are conducted on 8$\times$A100 80GB GPUs. Additional systems-profiling experiments in Appendix~\ref{app:extended_efficiency} use the hardware stack specified in each table: the pipeline and dual-GPU profiling experiments use 2$\times$A100 80GB GPUs, and the single-GPU scaling experiments use 1$\times$A100 80GB GPU. Each trainable configuration is repeated three times with different random seeds. Unless otherwise stated, reported benchmark values, deltas, win counts, and average scores are computed after averaging the three seed-specific outputs for each configuration. For SFT, we use a per-device batch size of 4 and a learning rate of $2{\times}10^{-5}$. LoPT uses a default auxiliary reconstruction weight of $\lambda_{\mathrm{aux}}{=}10.0$ and the same optimizer family as the matched E2E run. For GRPO, each policy-update group contains 4 prompts with 4 sampled responses per prompt. We train for 1 epoch with learning rate $5{\times}10^{-6}$, rollout temperature $0.7$, top-$p=0.95$, and clipping parameter $\varepsilon{=}0.2$. Unless otherwise noted, LoPT differs from the matched E2E baseline only by the midpoint stop-gradient boundary and the auxiliary reconstruction update for the first-half block.

\subsection{Evaluation Details}
\label{app:evaluation_details}
For downstream benchmark, we assess LoPT with lm-eval-harness~\citep{gao2023lmeval} under a unified protocol. The main evaluation benchmark includes MMLU (5-shot)~\citep{hendrycks2021mmlu}, IFEval~\citep{zhou2024ifeval}, ARC-Challenge (25-shot)~\citep{clark2018arc}, GSM8K (4-shot)~\citep{cobbe2021gsm8k}, HellaSwag (5-shot)~\citep{zellers2019hellaswag}, TruthfulQA MC2~\citep{lin2022truthfulqa}, and Winogrande (5-shot)~\citep{sakaguchi2021winogrande}. These benchmarks cover general knowledge, instruction following, reasoning, commonsense understanding, truthfulness, and mathematical problem solving. For trained checkpoints, all reported lm-eval scores are means over three independently trained checkpoints with different random seeds. We compute aggregate averages and method differences from these seed-averaged scores.

For training efficiency, we report peak GPU memory during training. For SFT, speed is measured as total training throughput in tokens/s. For GRPO, speed is measured as policy-update step time because rollout generation is unchanged between E2E and LoPT and is therefore not the source of the reported update-phase savings.

\section{Full SFT Results}
\label{app:full_sft}

This section provides the full SFT sweep behind the main-text SFT discussion. Table~\ref{tab:sft_quality_efficiency} shows a representative slice; here we report all four SFT datasets---Alpaca-52K, Tulu-3 SFT Mix 100K, Magpie-Pro 100K, and MetaMathQA 100K---for Qwen3-4B, Qwen2.5-7B-Instruct, and Llama-3.1-8B-Instruct. All rows use the same seven-benchmark lm-eval suite. Trained-method rows report means over three independent runs with different random seeds. The average deltas are unweighted averages over the seven reported seed-averaged benchmarks.

Base-model scores are reported once in Table~\ref{tab:base_scores_sft}; the per-dataset SFT tables below omit repeated base columns for compactness and focus on E2E--LoPT deltas. These appendix tables should be read together with the main-text claim: LoPT is most helpful when the SFT supervision is narrower than, or weakly aligned with, the held-out evaluation tasks, while math-aligned SFT produces more near-parity behavior.

\begin{table}[H]
\caption{\textbf{Base-model scores used for SFT comparisons.} Values are percentages under the same seven-benchmark lm-eval protocol as the SFT tables.}
\label{tab:base_scores_sft}
\centering
\small
\setlength{\tabcolsep}{4pt}
\begin{tabular}{lcccccccc}
\toprule
Model & MMLU & IFEval & ARC-C & GSM8K & HellaSwag & TruthfulQA & Winogrande & Avg. \\
\midrule
Qwen3-4B & 70.07 & 23.84 & 57.84 & 85.14 & 65.62 & 54.85 & 67.56 & 60.70 \\
Qwen2.5-7B & 74.17 & 57.30 & 66.72 & 83.47 & 79.35 & 64.75 & 74.51 & 71.47 \\
Llama-3.1-8B & 68.20 & 44.54 & 60.49 & 69.67 & 79.42 & 54.56 & 77.50 & 64.91 \\
\bottomrule
\end{tabular}
\end{table}

\begin{table}[H]
\caption{\textbf{SFT training statistics} for all 12 SFT configurations. Memory is peak memory per GPU on 8$\times$A100 80GB GPUs. Throughput is total training throughput in tokens/s. Reported values are averaged over three independent runs with different random seeds.}
\label{tab:sft_training}
\centering
\footnotesize
\setlength{\tabcolsep}{3pt}
\resizebox{0.9\textwidth}{!}{
\begin{tabular}{llcccccc}
\toprule
& & \multicolumn{2}{c}{Peak Mem (GB/GPU)} & \multicolumn{2}{c}{Throughput (tok/s)} & \multicolumn{2}{c}{Time (s)} \\
\cmidrule(lr){3-4}\cmidrule(lr){5-6}\cmidrule(lr){7-8}
Dataset & Model & E2E & LoPT & E2E & LoPT & E2E & LoPT \\
\midrule
\multirow{3}{*}{Alpaca-52K}
& Qwen3-4B & 31.5 & 21.9\,($-$30\%) & 21{,}020 & 21{,}469 & 3{,}800 & 3{,}720 \\
& Qwen2.5-7B & 57.7 & 37.2\,($-$36\%) & 15{,}561 & 15{,}971 & 5{,}133 & 5{,}001 \\
& Llama-3.1-8B & 60.6 & 38.9\,($-$36\%) & 14{,}483 & 14{,}882 & 5{,}515 & 5{,}367 \\
\midrule
\multirow{3}{*}{Tulu-3 SFT Mix 100K}
& Qwen3-4B & 42.6 & 32.9\,($-$23\%) & 25{,}323 & 30{,}162 & 8{,}087 & 6{,}790 \\
& Qwen2.5-7B & 63.5 & 45.5\,($-$28\%) & 21{,}291 & 23{,}951 & 9{,}619 & 8{,}551 \\
& Llama-3.1-8B & 64.1 & 45.2\,($-$29\%) & 19{,}185 & 21{,}747 & 10{,}675 & 9{,}418 \\
\midrule
\multirow{3}{*}{Magpie-Pro 100K}
& Qwen3-4B & 33.0 & 25.2\,($-$24\%) & 24{,}963 & 26{,}571 & 4{,}102 & 3{,}854 \\
& Qwen2.5-7B & 58.7 & 38.4\,($-$35\%) & 19{,}345 & 20{,}581 & 5{,}293 & 4{,}975 \\
& Llama-3.1-8B & 61.4 & 39.9\,($-$35\%) & 17{,}833 & 19{,}083 & 5{,}742 & 5{,}366 \\
\midrule
\multirow{3}{*}{MetaMathQA 100K}
& Qwen3-4B & 33.0 & 25.2\,($-$24\%) & 25{,}019 & 26{,}670 & 4{,}093 & 3{,}840 \\
& Qwen2.5-7B & 58.7 & 38.4\,($-$35\%) & 19{,}567 & 20{,}511 & 5{,}233 & 4{,}992 \\
& Llama-3.1-8B & 61.4 & 39.9\,($-$35\%) & 18{,}028 & 19{,}194 & 5{,}680 & 5{,}335 \\
\bottomrule
\end{tabular}
}
\end{table}

\paragraph{Efficiency pattern across the full SFT sweep.}
Table~\ref{tab:sft_training} shows that the memory reduction is consistent across all SFT datasets and model families: LoPT saves 23--36\% peak memory relative to the matched E2E run. Throughput also improves in every measured configuration, ranging from modest gains on Alpaca-52K to larger gains on the longer 100K-example settings. This supports the systems claim that the auxiliary reconstruction branch does not erase the savings from truncating the task-loss backward graph.

\begin{table}[H]
\caption{\textbf{Alpaca-52K SFT results} with per-benchmark deltas (\%). Reported values are averaged over three independent runs with different random seeds.}
\label{tab:app_alpaca}
\centering
\footnotesize
\setlength{\tabcolsep}{2.5pt}
\resizebox{0.9\textwidth}{!}{
\begin{tabular}{l|ccc|ccc|ccc}
\toprule
& \multicolumn{3}{c|}{Qwen3-4B} & \multicolumn{3}{c|}{Qwen2.5-7B} & \multicolumn{3}{c}{Llama-3.1-8B} \\
Benchmark & E2E & LoPT & $\Delta$ & E2E & LoPT & $\Delta$ & E2E & LoPT & $\Delta$ \\
\midrule
MMLU & \textbf{70.12} & 69.99 & \llneutral{$-0.13$} & 73.09 & \textbf{74.08} & \llneutral{$+0.99$} & \textbf{68.67} & 68.48 & \llneutral{$-0.19$} \\
IFEval & 25.87 & \textbf{29.39} & \llwin{$\mathbf{+3.52}$} & 40.30 & \textbf{41.22} & \llneutral{$+0.92$} & 30.50 & \textbf{46.03} & \llwin{$\mathbf{+15.53}$} \\
ARC-C & 61.77 & \textbf{61.95} & \llneutral{$+0.18$} & 60.58 & \textbf{63.96} & \llwin{$\mathbf{+3.38}$} & 57.59 & \textbf{60.41} & \llwin{$+2.82$} \\
GSM8K & 79.45 & \textbf{84.62} & \llwin{$\mathbf{+5.17}$} & 74.22 & \textbf{79.38} & \llwin{$\mathbf{+5.16}$} & 46.62 & \textbf{72.10} & \llwin{$\mathbf{+25.48}$} \\
HellaSwag & 69.80 & \textbf{69.97} & \llneutral{$+0.17$} & 78.55 & \textbf{79.46} & \llneutral{$+0.91$} & 78.31 & \textbf{79.38} & \llwin{$+1.07$} \\
TruthfulQA & 53.07 & \textbf{53.91} & \llneutral{$+0.84$} & 52.28 & \textbf{54.86} & \llwin{$+2.58$} & 47.39 & \textbf{52.67} & \llwin{$\mathbf{+5.28}$} \\
Winogrande & \textbf{68.19} & 67.69 & \llneutral{$-0.50$} & \textbf{75.77} & 75.63 & \llneutral{$-0.14$} & 76.09 & \textbf{76.64} & \llneutral{$+0.55$} \\
\midrule
LoPT Wins & \multicolumn{3}{c|}{5/7, avg \llwin{$+1.32\pp$}} & \multicolumn{3}{c|}{6/7, avg \llwin{$+1.97\pp$}} & \multicolumn{3}{c}{6/7, avg \llwin{$\mathbf{+7.22\pp}$}} \\
\bottomrule
\end{tabular}
}
\end{table}

\paragraph{Alpaca-52K.}
Table~\ref{tab:app_alpaca} shows the clearest held-out preservation pattern. Alpaca-52K is a broad instruction dataset and does not directly rehearse several of the evaluated skills. LoPT improves the average over E2E for all three model families, with the largest effect on Llama-3.1-8B-Instruct ($+7.22\pp$). The largest individual gains are on GSM8K and IFEval, indicating that full-depth E2E instruction tuning can strongly disturb capabilities not emphasized by the SFT data, while limiting task-gradient reach helps reduce this degradation.

\begin{table}[H]
\caption{\textbf{Tulu-3 SFT Mix 100K SFT results} with per-benchmark deltas (\%). Reported values are averaged over three independent runs with different random seeds.}
\label{tab:app_tulu3}
\centering
\footnotesize
\setlength{\tabcolsep}{2.5pt}
\resizebox{0.9\textwidth}{!}{
\begin{tabular}{l|ccc|ccc|ccc}
\toprule
& \multicolumn{3}{c|}{Qwen3-4B} & \multicolumn{3}{c|}{Qwen2.5-7B} & \multicolumn{3}{c}{Llama-3.1-8B} \\
Benchmark & E2E & LoPT & $\Delta$ & E2E & LoPT & $\Delta$ & E2E & LoPT & $\Delta$ \\
\midrule
MMLU & 69.35 & \textbf{69.80} & \llneutral{$+0.45$} & 73.64 & \textbf{74.16} & \llneutral{$+0.52$} & 65.35 & \textbf{67.54} & \llwin{$\mathbf{+2.19}$} \\
IFEval & \textbf{21.63} & 20.59 & \llloss{$-1.04$} & 33.09 & \textbf{39.74} & \llwin{$\mathbf{+6.65}$} & 23.29 & \textbf{37.34} & \llwin{$\mathbf{+14.05}$} \\
ARC-C & 60.75 & \textbf{61.60} & \llneutral{$+0.85$} & 60.49 & \textbf{63.40} & \llwin{$\mathbf{+2.91}$} & 57.94 & \textbf{60.41} & \llwin{$\mathbf{+2.47}$} \\
GSM8K & 77.10 & \textbf{79.45} & \llwin{$+2.35$} & 51.55 & \textbf{68.31} & \llwin{$\mathbf{+16.76}$} & 63.91 & \textbf{74.83} & \llwin{$\mathbf{+10.92}$} \\
HellaSwag & \textbf{70.93} & 70.58 & \llneutral{$-0.35$} & 79.08 & \textbf{79.57} & \llneutral{$+0.49$} & 78.68 & \textbf{79.19} & \llneutral{$+0.51$} \\
TruthfulQA & 53.14 & \textbf{53.77} & \llneutral{$+0.63$} & 50.66 & \textbf{55.51} & \llwin{$\mathbf{+4.85}$} & 48.83 & \textbf{54.17} & \llwin{$\mathbf{+5.34}$} \\
Winogrande & \textbf{69.53} & 68.67 & \llneutral{$-0.86$} & \textbf{75.93} & 74.51 & \llloss{$-1.42$} & \textbf{78.22} & 77.11 & \llloss{$-1.11$} \\
\midrule
LoPT Wins & \multicolumn{3}{c|}{4/7, avg \llneutral{$+0.29\pp$}} & \multicolumn{3}{c|}{6/7, avg \llwin{$\mathbf{+4.39\pp}$}} & \multicolumn{3}{c}{6/7, avg \llwin{$\mathbf{+4.91\pp}$}} \\
\bottomrule
\end{tabular}
}
\end{table}

\paragraph{Tulu-3 SFT Mix 100K.}
Table~\ref{tab:app_tulu3} shows a mixed but still positive pattern. Qwen3-4B is near parity on average ($+0.29\pp$), whereas Qwen2.5-7B-Instruct and Llama-3.1-8B-Instruct show larger gains of $+4.39\pp$ and $+4.91\pp$. The improvements are concentrated on IFEval, GSM8K, ARC-C, and TruthfulQA, while Winogrande can slightly decline. We therefore interpret this setting as evidence that LoPT mainly protects broad held-out behavior, rather than uniformly improving every benchmark.

\begin{table}[H]
\caption{\textbf{Magpie-Pro 100K SFT results} with per-benchmark deltas (\%). Reported values are averaged over three independent runs with different random seeds.}
\label{tab:app_magpie}
\centering
\footnotesize
\setlength{\tabcolsep}{2.5pt}
\resizebox{0.9\textwidth}{!}{
\begin{tabular}{l|ccc|ccc|ccc}
\toprule
& \multicolumn{3}{c|}{Qwen3-4B} & \multicolumn{3}{c|}{Qwen2.5-7B} & \multicolumn{3}{c}{Llama-3.1-8B} \\
Benchmark & E2E & LoPT & $\Delta$ & E2E & LoPT & $\Delta$ & E2E & LoPT & $\Delta$ \\
\midrule
MMLU & 69.12 & \textbf{69.69} & \llneutral{$+0.57$} & 73.28 & \textbf{73.86} & \llneutral{$+0.58$} & 66.67 & \textbf{67.80} & \llwin{$+1.13$} \\
IFEval & 25.88 & \textbf{29.39} & \llwin{$\mathbf{+3.51}$} & 36.23 & \textbf{46.03} & \llwin{$\mathbf{+9.80}$} & \textbf{46.21} & 44.62 & \llloss{$-1.59$} \\
ARC-C & 61.77 & 61.77 & $0.00$ & 62.20 & \textbf{64.68} & \llwin{$\mathbf{+2.48}$} & 60.84 & \textbf{61.77} & \llneutral{$+0.93$} \\
GSM8K & 79.45 & \textbf{83.24} & \llwin{$\mathbf{+3.79}$} & 68.20 & \textbf{80.44} & \llwin{$\mathbf{+12.24}$} & \textbf{72.48} & 71.89 & \llneutral{$-0.59$} \\
HellaSwag & \textbf{70.80} & 69.98 & \llneutral{$-0.82$} & 79.00 & \textbf{79.69} & \llneutral{$+0.69$} & 78.08 & \textbf{79.14} & \llwin{$+1.06$} \\
TruthfulQA & 53.03 & \textbf{53.81} & \llneutral{$+0.78$} & \textbf{59.89} & 59.80 & \llneutral{$-0.09$} & 54.65 & \textbf{54.66} & \llneutral{$+0.01$} \\
Winogrande & \textbf{68.19} & 66.69 & \llloss{$-1.50$} & \textbf{74.90} & 74.19 & \llneutral{$-0.71$} & 76.48 & \textbf{76.95} & \llneutral{$+0.47$} \\
\midrule
LoPT Wins & \multicolumn{3}{c|}{4/7, avg \llneutral{$+0.90\pp$}} & \multicolumn{3}{c|}{5/7, avg \llwin{$\mathbf{+3.57\pp}$}} & \multicolumn{3}{c}{5/7, avg \llneutral{$+0.20\pp$}} \\
\bottomrule
\end{tabular}
}
\end{table}

\paragraph{Magpie-Pro 100K.}
Table~\ref{tab:app_magpie} again shows that the benefit depends on model and benchmark. Qwen2.5-7B-Instruct gains $+3.57\pp$ on average, driven by large improvements on IFEval and GSM8K. Qwen3-4B has a smaller positive average, and Llama-3.1-8B-Instruct is close to parity. This is useful as a stress case: LoPT is not a universal per-benchmark improvement rule, but it tends to reduce large negative transfers when E2E post-training pushes the model away from held-out skills.

\begin{table}[H]
\caption{\textbf{MetaMathQA 100K SFT results} with per-benchmark deltas (\%). Reported values are averaged over three independent runs with different random seeds.}
\label{tab:app_metamath}
\centering
\footnotesize
\setlength{\tabcolsep}{2.5pt}
\resizebox{0.9\textwidth}{!}{
\begin{tabular}{l|ccc|ccc|ccc}
\toprule
& \multicolumn{3}{c|}{Qwen3-4B} & \multicolumn{3}{c|}{Qwen2.5-7B} & \multicolumn{3}{c}{Llama-3.1-8B} \\
Benchmark & E2E & LoPT & $\Delta$ & E2E & LoPT & $\Delta$ & E2E & LoPT & $\Delta$ \\
\midrule
MMLU & \textbf{70.18} & 70.17 & \llneutral{$-0.01$} & 74.13 & \textbf{74.23} & \llneutral{$+0.10$} & 65.29 & \textbf{67.94} & \llwin{$\mathbf{+2.65}$} \\
IFEval & 20.89 & \textbf{22.26} & \llwin{$+1.37$} & \textbf{51.76} & 51.46 & \llneutral{$-0.30$} & 34.94 & \textbf{43.99} & \llwin{$\mathbf{+9.05}$} \\
ARC-C & 63.57 & \textbf{63.64} & \llneutral{$+0.07$} & 63.91 & \textbf{64.33} & \llneutral{$+0.42$} & 58.70 & \textbf{60.92} & \llwin{$\mathbf{+2.22}$} \\
GSM8K & 81.20 & \textbf{81.21} & \llneutral{$+0.01$} & 86.80 & \textbf{87.21} & \llneutral{$+0.41$} & 76.79 & \textbf{79.27} & \llwin{$\mathbf{+2.48}$} \\
HellaSwag & 72.43 & \textbf{72.53} & \llneutral{$+0.10$} & 80.85 & \textbf{80.97} & \llneutral{$+0.12$} & 77.70 & \textbf{79.71} & \llwin{$\mathbf{+2.01}$} \\
TruthfulQA & 54.49 & \textbf{54.55} & \llneutral{$+0.06$} & 62.63 & \textbf{62.66} & \llneutral{$+0.03$} & 52.06 & \textbf{52.93} & \llneutral{$+0.87$} \\
Winogrande & 68.69 & \textbf{68.88} & \llneutral{$+0.19$} & 71.51 & \textbf{74.03} & \llwin{$\mathbf{+2.52}$} & 76.56 & \textbf{76.64} & \llneutral{$+0.08$} \\
\midrule
LoPT Wins & \multicolumn{3}{c|}{6/7, avg \llneutral{$+0.26\pp$}} & \multicolumn{3}{c|}{6/7, avg \llneutral{$+0.47\pp$}} & \multicolumn{3}{c}{7/7, avg \llwin{$+2.77\pp$}} \\
\bottomrule
\end{tabular}
}
\end{table}

\paragraph{MetaMathQA 100K.}
Table~\ref{tab:app_metamath} is more task-aligned with GSM8K-style reasoning, so the contrast between LoPT and E2E is smaller for Qwen3-4B and Qwen2.5-7B-Instruct. The Llama-3.1-8B-Instruct block still improves by $+2.77\pp$ on average, with gains across all seven benchmarks. This supports the main-text interpretation that LoPT is most valuable when E2E produces held-out degradation, while in well-aligned settings it usually remains near parity and preserves the efficiency advantage.

\FloatBarrier
\section{Full GRPO Results}
\label{app:full_grpo}

This section provides the full GRPO results behind the main-text GRPO discussion in Section~\ref{sec:capability}. Table~\ref{tab:grpo_quality_efficiency} includes the matched-model GSM8K GRPO slice used in the main text; here we provide the complete GSM8K and NuminaMath GRPO tables under the same seven-benchmark lm-eval protocol used for SFT. The GRPO quality deltas are intentionally interpreted conservatively: many are within one percentage point and should be read as near parity unless the sign and magnitude are consistent across related metrics.

The training-statistics table reports the measured GRPO policy-update phase. Rollout generation is unchanged between E2E and LoPT, so the speed numbers isolate the policy optimization step rather than full RL wall-clock time.

\begin{table}[H]
\caption{\textbf{GRPO training statistics} for the measured GSM8K GRPO configurations. Memory is peak memory per GPU on 8$\times$A100 80GB GPUs, and step time is s/step for the policy-update phase. Reported values are averaged over three independent runs with different random seeds.}
\label{tab:grpo_training}
\centering
\footnotesize
\setlength{\tabcolsep}{3pt}
\begin{tabular}{llcccccc}
\toprule
& & \multicolumn{2}{c}{Peak Mem (GB/GPU)} & \multicolumn{2}{c}{Step Time (s)} & \multicolumn{2}{c}{Step-Time Reduction} \\
\cmidrule(lr){3-4}\cmidrule(lr){5-6}\cmidrule(lr){7-8}
Dataset & Model & E2E & LoPT & E2E & LoPT & Abs. (s) & Rel. \\
\midrule
GSM8K & Qwen3-4B & 42.55 & 32.89 & 1.42 & 1.32 & 0.10 & $+7\%$ \\
GSM8K & Qwen2.5-7B & 75.39 & 59.66 & 1.71 & 1.56 & 0.15 & $+9\%$ \\
GSM8K & Llama-3.1-8B & 62.62 & 47.54 & 2.10 & 1.65 & 0.45 & $+21\%$ \\
\bottomrule
\end{tabular}
\end{table}

\paragraph{GRPO update cost.}
Reported values are averaged over three independent runs with different random seeds.
Table~\ref{tab:grpo_training} shows consistent update-phase savings across all three seed-averaged GSM8K GRPO model configurations. Peak memory decreases by 21--24\%, and update step time reduces by 7--21\%. The largest speedup appears for Llama-3.1-8B-Instruct, where stopping the task backward path below the midpoint removes a larger amount of backward computation relative to the measured policy-update workload.

\subsection{GSM8K GRPO}
\label{app:grpo_gsm8k}

\begin{table}[H]
\caption{\textbf{Qwen3-4B GSM8K GRPO} --- seven-benchmark results (\%). Reported values are averaged over three independent runs with different random seeds.}
\label{tab:grpo_qwen3}
\centering
\small
\begin{tabular}{lcccc}
\toprule
Benchmark & Base & E2E-GRPO & LoPT-GRPO & $\Delta$ \\
\midrule
MMLU & 70.07 & 70.01 & \textbf{70.10} & \llneutral{$+0.09$} \\
IFEval & 23.84 & \textbf{24.03} & 23.84 & \llneutral{$-0.19$} \\
ARC-C & 57.84 & 62.03 & \textbf{62.29} & \llneutral{$+0.26$} \\
GSM8K & 85.14 & 88.35 & \textbf{88.63} & \llneutral{$+0.28$} \\
HellaSwag & 65.62 & 68.76 & \textbf{69.95} & \llwin{$\mathbf{+1.19}$} \\
TruthfulQA & 54.85 & \textbf{54.78} & 54.62 & \llneutral{$-0.16$} \\
Winogrande & 67.56 & \textbf{67.17} & 66.85 & \llneutral{$-0.32$} \\
\midrule
LoPT Wins & & \multicolumn{2}{c}{4/7} & Avg: \llneutral{$+0.16\pp$} \\
\bottomrule
\end{tabular}
\end{table}

\begin{table}[H]
\caption{\textbf{Qwen2.5-7B GSM8K GRPO} --- seven-benchmark results (\%). Reported values are averaged over three independent runs with different random seeds.}
\label{tab:grpo_qwen7b}
\centering
\small
\begin{tabular}{lcccc}
\toprule
Benchmark & Base & E2E-GRPO & LoPT-GRPO & $\Delta$ \\
\midrule
MMLU & 74.17 & \textbf{74.22} & 74.18 & \llneutral{$-0.04$} \\
IFEval & 57.30 & 56.01 & \textbf{59.70} & \llwin{$\mathbf{+3.69}$} \\
ARC-C & \textbf{66.72} & \textbf{66.72} & 66.38 & \llneutral{$-0.34$} \\
GSM8K & 83.47 & 88.70 & \textbf{89.09} & \llneutral{$+0.39$} \\
HellaSwag & 79.35 & 81.23 & \textbf{81.40} & \llneutral{$+0.17$} \\
TruthfulQA & 64.75 & 64.52 & \textbf{64.83} & \llneutral{$+0.31$} \\
Winogrande & 74.51 & 75.22 & \textbf{75.40} & \llneutral{$+0.18$} \\
\midrule
LoPT Wins & & \multicolumn{2}{c}{5/7} & Avg: \llneutral{$+0.62\pp$} \\
\bottomrule
\end{tabular}
\end{table}

\begin{table}[H]
\caption{\textbf{Llama-3.1-8B GSM8K GRPO} --- seven-benchmark results (\%). Reported values are averaged over three independent runs with different random seeds.}
\label{tab:grpo_llama}
\centering
\small
\begin{tabular}{lcccc}
\toprule
Benchmark & Base & E2E-GRPO & LoPT-GRPO & $\Delta$ \\
\midrule
MMLU & 68.20 & 68.44 & \textbf{68.55} & \llneutral{$+0.11$} \\
IFEval & 44.54 & \textbf{47.69} & 47.03 & \llneutral{$-0.66$} \\
ARC-C & 60.49 & \textbf{61.69} & 61.43 & \llneutral{$-0.26$} \\
GSM8K & 69.67 & 82.02 & \textbf{82.40} & \llneutral{$+0.38$} \\
HellaSwag & 79.42 & 80.39 & \textbf{80.47} & \llneutral{$+0.08$} \\
TruthfulQA & 54.56 & \textbf{54.53} & 54.11 & \llneutral{$-0.42$} \\
Winogrande & 77.50 & 76.80 & \textbf{77.35} & \llneutral{$+0.55$} \\
\midrule
LoPT Wins & & \multicolumn{2}{c}{4/7} & Avg: \llneutral{$-0.03\pp$} \\
\bottomrule
\end{tabular}
\end{table}

\paragraph{GSM8K GRPO.}
Across Tables~\ref{tab:grpo_qwen3}--\ref{tab:grpo_llama}, LoPT remains close to E2E on the optimized GSM8K metric, with gains below half a point for all three models. This is expected because GSM8K is both the training data and the main optimized reasoning task. The held-out columns show small preservation effects in some cases, most notably IFEval for Qwen2.5-7B-Instruct, but the overall conclusion is near-parity quality with lower policy-update cost.

\FloatBarrier
\subsection{NuminaMath GRPO}
\label{app:grpo_numina}

\begin{table}[H]
\caption{\textbf{Qwen2.5-7B NuminaMath GRPO} --- seven-benchmark results (\%). Reported values are averaged over three independent runs with different random seeds.}
\label{tab:grpo_numina_qwen7b}
\centering
\small
\begin{tabular}{lcccc}
\toprule
Benchmark & Base & E2E-GRPO & LoPT-GRPO & $\Delta$ \\
\midrule
MMLU & 74.17 & 74.14 & \textbf{74.21} & \llneutral{$+0.07$} \\
IFEval & 57.30 & 58.23 & \textbf{60.26} & \llwin{$\mathbf{+2.03}$} \\
ARC-C & 66.72 & \textbf{66.30} & 65.87 & \llneutral{$-0.43$} \\
GSM8K & 83.47 & \textbf{86.58} & 86.35 & \llneutral{$-0.23$} \\
HellaSwag & 79.35 & 81.41 & \textbf{81.63} & \llneutral{$+0.22$} \\
TruthfulQA & 64.75 & 64.16 & \textbf{64.43} & \llneutral{$+0.27$} \\
Winogrande & 74.51 & 73.16 & \textbf{74.98} & \llwin{$\mathbf{+1.82}$} \\
\midrule
LoPT Wins & & \multicolumn{2}{c}{5/7} & Avg: \llneutral{$+0.54\pp$} \\
\bottomrule
\end{tabular}
\end{table}

\begin{table}[H]
\caption{\textbf{Llama-3.1-8B NuminaMath GRPO} --- seven-benchmark results (\%). Reported values are averaged over three independent runs with different random seeds.}
\label{tab:grpo_numina_llama}
\centering
\small
\begin{tabular}{lcccc}
\toprule
Benchmark & Base & E2E-GRPO & LoPT-GRPO & $\Delta$ \\
\midrule
MMLU & 68.20 & 67.25 & \textbf{68.34} & \llwin{$\mathbf{+1.09}$} \\
IFEval & 44.54 & 42.88 & \textbf{43.07} & \llneutral{$+0.19$} \\
ARC-C & 60.49 & 57.68 & \textbf{61.01} & \llwin{$\mathbf{+3.33}$} \\
GSM8K & 69.67 & 81.97 & \textbf{84.31} & \llwin{$\mathbf{+2.34}$} \\
HellaSwag & 79.42 & 79.71 & \textbf{79.84} & \llneutral{$+0.13$} \\
TruthfulQA & 54.56 & 53.88 & \textbf{53.96} & \llneutral{$+0.08$} \\
Winogrande & 77.50 & 77.03 & \textbf{77.40} & \llneutral{$+0.37$} \\
\midrule
LoPT Wins & & \multicolumn{2}{c}{\textbf{7/7}} & Avg: \llwin{$\mathbf{+1.08\pp}$} \\
\bottomrule
\end{tabular}
\end{table}

\paragraph{NuminaMath GRPO.}
Tables~\ref{tab:grpo_numina_qwen7b} and~\ref{tab:grpo_numina_llama} provide an additional math-RL check. Qwen2.5-7B-Instruct is near parity on the training-aligned GSM8K metric and gains on IFEval and Winogrande, while Llama-3.1-8B-Instruct improves on all seven reported benchmarks. Because this is a smaller GRPO sweep than the SFT study, we use it as supporting evidence that LoPT does not harm GRPO adaptation and can preserve held-out behavior in some math-RL settings.

\FloatBarrier
\section{Additional Ablation on the First-Half-Block Objective}
\label{app:local_objective_ablation}

LoPT uses a feature-reconstruction objective for the first-half block
instead of asking the first-half block to solve a local next-token
prediction task. This design choice is important: the first-half block is
intended to remain trainable and preserve a useful interface for the
second-half block, rather than to become a standalone local decoder. To
verify this choice, we compare the default reconstruction objective with
local next-token-prediction (NTP) alternatives in
Table~\ref{tab:local_objective_ablation}. This ablation is conducted on
Qwen2.5-7B-Instruct fine-tuned on MetaMathQA 100K under the SFT setting.

Here, \textbf{NTP} denotes a local next-token-prediction auxiliary loss
attached to the first-half block. \textbf{NTP+Recon} uses both local NTP
and feature reconstruction; the reconstruction weight is kept the same
as LoPT, while the NTP loss is multiplied by \(0.01\). \textbf{Recon}
denotes the default LoPT objective. All trainable variants use the same
first-half/second-half split and gradient checkpointing (GC). Values are
percentages averaged over three runs.

\begin{table}[htbp]
\caption{\textbf{Ablation on first-half-block local objectives.}
We compare local next-token prediction (NTP), NTP plus feature
reconstruction, and the default feature-reconstruction objective on
Qwen2.5-7B-Instruct fine-tuned on MetaMathQA 100K under SFT. Reported benchmark scores are averaged over three independent runs with different random seeds. Bold numbers indicate the best local-learning objective among NTP, NTP+Recon, and Recon.}
\label{tab:local_objective_ablation}
\centering
\scriptsize
\setlength{\tabcolsep}{4.5pt}
\renewcommand{\arraystretch}{1.08}
\resizebox{\textwidth}{!}{
\begin{tabular}{lccccc}
\toprule
Benchmark
& NTP / LL+GC
& NTP+Recon / LL+GC
& Recon / LL+GC
& E2E+GC
& Base \\
\midrule
MMLU
& 54.63
& 73.78
& \textbf{74.23}
& 74.13
& 74.17 \\
IFEval strict
& 36.49
& 49.72
& \textbf{51.46}
& 51.76
& 57.30 \\
ARC-Challenge
& 50.55
& \textbf{64.59}
& 64.33
& 63.91
& 66.72 \\
GSM8K
& 42.97
& 77.26
& \textbf{87.21}
& 86.80
& 83.47 \\
HellaSwag
& 65.36
& 80.69
& \textbf{80.97}
& 80.85
& 79.35 \\
TruthfulQA MC2
& 49.71
& 61.18
& \textbf{62.66}
& 62.63
& 64.75 \\
Winogrande
& 60.83
& 72.93
& \textbf{74.03}
& 71.51
& 74.51 \\
\midrule
Avg.
& 51.51
& 68.59
& \textbf{70.70}
& 70.23
& 71.47 \\
\bottomrule
\end{tabular}
}
\vspace{-0.5em}
\begin{flushleft}
\footnotesize{
LL denotes the same first-half/second-half local-learning split as LoPT,
and GC denotes gradient checkpointing. NTP is a local next-token-prediction
auxiliary loss attached to the first-half block. In NTP+Recon, the
reconstruction weight is unchanged from LoPT and the NTP loss is multiplied
by \(0.01\). E2E+GC and Base are included as references.
}
\end{flushleft}
\end{table}

The results show that a local task-level decoding objective is a poor
replacement for feature reconstruction in the first-half block. Pure
local NTP severely degrades performance, reducing the seven-benchmark
average to \(51.51\), nearly \(19.2\pp\) below the default reconstruction
objective and about \(20.0\pp\) below the base checkpoint. The degradation
is broad rather than isolated: local NTP substantially lowers MMLU,
ARC-Challenge, GSM8K, TruthfulQA, and Winogrande, and the GSM8K score
drops to \(42.97\), far below both E2E+GC and Recon. This indicates that
forcing the first-half block to perform next-token prediction locally is
too intrusive for the part of the model that LoPT intends to protect from
direct task-gradient pressure.

Adding feature reconstruction partially stabilizes the local NTP
objective, but does not remove the conflict. NTP+Recon recovers the
average score from \(51.51\) to \(68.59\), yet it remains \(2.11\pp\)
below Recon alone and is especially weaker on the training-aligned GSM8K
evaluation (\(77.26\) vs. \(87.21\)). This gap is notable because the NTP
loss is already down-weighted by \(0.01\). Even a small local decoding
loss can therefore pull the first-half block toward a task-predictive
representation that is less suitable as the interface consumed by the
second-half block.

The default feature-reconstruction objective gives the strongest
local-learning variant, winning six of seven benchmarks among the three
first-half-block objective choices and slightly outperforming E2E+GC on
the seven-benchmark average. This supports the central design of LoPT:
the first-half block should receive a feature-level maintenance signal
rather than a local task-solving signal. Feature reconstruction keeps
\(k_1\) trainable without directly imposing downstream token-level
supervision on it, while the standard SFT objective remains concentrated
in \(k_2\).
\section{Extended Efficiency Analysis}
\label{app:extended_efficiency}

This section expands the systems analysis in Section~\ref{sec:capability}. All memory values in this section are peak GPU memory per GPU under the specified hardware stack; relative memory savings are computed only between matched E2E and LoPT runs using the same hardware, batch size, sequence length, and parallelism setting. The key point is that LoPT changes the gradient route, not the optimizer state or parameterization. It reduces task-backward activation lifetime below the midpoint boundary, while methods such as gradient checkpointing, ZeRO, LoRA, and pipeline parallelism act through different mechanisms. The tables below therefore test whether the LoPT efficiency gain persists when these common training-stack choices are already enabled.

\subsection{Matched-Stack Profiling}
\label{app:stacking_profiles}

\paragraph{Matched-stack interpretation.}
Table~\ref{tab:stacking} shows that LoPT continues to save memory under gradient checkpointing, LoRA, and ZeRO. The relative memory gain is largest for gradient checkpointing and LoRA, and smaller but still positive under ZeRO-1/2/3 because ZeRO already reduces optimizer and parameter-state memory. Throughput also improves in all matched stacks. This supports the claim that gradient routing is orthogonal to both parameter-efficient adaptation and system-level memory partitioning.


\begin{table}[H]
\caption{\textbf{Compatibility with standard training stacks.}
Efficiency values are averaged over three measured runs under the specified hardware stack.
Memory is peak GPU memory per GPU in GB/GPU, and throughput is total training throughput in tokens/s.}
\label{tab:stacking}
\centering
\scriptsize
\setlength{\tabcolsep}{3pt}
\resizebox{\textwidth}{!}{
\begin{tabular}{llcccccc}
\toprule
Stack & Profiling setting & E2E Mem & LoPT Mem & Mem Saved & E2E Throughput & LoPT Throughput & Speed $\Delta$ \\
\midrule
Gradient Checkpointing
& Qwen2.5-7B, 8$\times$A100 80GB, bs=4, seq=1024
& 58.70\,GB & \textbf{38.40\,GB} & $\mathbf{-35\%}$
& 18{,}894 & \textbf{19{,}796} & $+5\%$ \\

LoRA (rank=32)
& Qwen2.5-7B, 8$\times$A100 80GB, bs=4, seq=1024
& 68.37\,GB & \textbf{43.88\,GB} & $\mathbf{-36\%}$
& 28{,}794 & \textbf{30{,}524} & $+6\%$ \\

DeepSpeed-ZeRO1
& Qwen2.5-7B, 8$\times$A100 80GB, bs=4, seq=1024
& 40.38\,GB & \textbf{33.05\,GB} & $\mathbf{-18\%}$
& 21{,}904 & \textbf{22{,}947} & $+5\%$ \\

DeepSpeed-ZeRO2
& Qwen2.5-7B, 8$\times$A100 80GB, bs=4, seq=1024
& 33.84\,GB & \textbf{28.61\,GB} & $\mathbf{-15\%}$
& 21{,}389 & \textbf{22{,}987} & $+7\%$ \\

DeepSpeed-ZeRO3
& Qwen2.5-7B, 8$\times$A100 80GB, bs=4, seq=1024
& 25.18\,GB & \textbf{20.42\,GB} & $\mathbf{-19\%}$
& 21{,}066 & \textbf{22{,}413} & $+6\%$ \\
\bottomrule
\end{tabular}
}
\end{table}





\begin{table}[H]
\caption{\textbf{Pipeline-parallel layout check.} Qwen3-4B, 2$\times$A100 80GB, batch size 1, sequence length 2048; memory in GB/GPU and speed $\Delta$ relative to DDP E2E. Efficiency values are averaged over three measured runs under the specified hardware stack.}
\label{tab:pp_layout}
\centering
\small
\begin{tabular}{lcccc}
\toprule
Method & Profiling setting & Peak Mem. & Mem Saved & Speed $\Delta$ \\
\midrule
DDP E2E & Qwen3-4B, 2$\times$A100 80GB, bs=1, seq=2048 & 49.6\,GB & --- & --- \\
PP-LoPT & Qwen3-4B, 2$\times$A100 80GB, bs=1, seq=2048 & \textbf{23.2\,GB} & $\mathbf{-53\%}$ & $\mathbf{+22\%}$ \\
\bottomrule
\end{tabular}
\end{table}

The pipeline-parallel layout check shows that the LoPT boundary can be aligned with a device boundary: the first device handles the reconstruction-updated block and the second device handles the task-updated block. In this profiled Qwen3-4B setting, PP-LoPT reduces memory by 53\% and improves speed by 22\% relative to DDP E2E. This table is not meant to claim an optimal pipeline schedule; it verifies that the midpoint gradient boundary can be exploited by a simple two-stage layout.

\subsection{Dual-GPU Efficiency}
\label{app:dual_gpu_efficiency}

\begin{table}[H]
\caption{\textbf{Dual-GPU efficiency} (2$\times$A100 80GB, Qwen3-4B, bs=1, seq=2048). Format: peak memory / throughput. Efficiency values are averaged over three measured runs under the specified hardware stack.}
\label{tab:dual_gpu}
\centering
\small
\begin{tabular}{lcccc}
\toprule
Method & Parallelism & Peak Mem (GB/GPU) & Throughput (tok/s/GPU) & Mem.\ Eff. \\
\midrule
E2E & DDP & 49.6 & 2{,}735 & 55.1 \\
E2E+GC & DDP & 31.5 & 2{,}256 & 71.6 \\
LoPT & PP & \textbf{23.2} & \textbf{3{,}343} & \textbf{144.1} \\
LoPT+GC & PP & \textbf{13.6} & 2{,}769 & \textbf{203.6} \\
\bottomrule
\end{tabular}
\end{table}

Table~\ref{tab:dual_gpu} separates two regimes. LoPT without checkpointing is the fastest option among the listed methods, while LoPT+GC gives the lowest memory and highest memory efficiency. This illustrates the expected trade-off: checkpointing can further reduce memory, but it pays recomputation cost; LoPT reduces the task backward graph directly and can therefore improve speed when checkpointing is not the bottleneck.

Pipeline-parallel LoPT is 22\% faster than DDP E2E while using 53\% less memory. LoPT+GC achieves 3.7$\times$ the memory efficiency of DDP E2E in this setting.

\subsection{Single-GPU Scaling Across Batch and Sequence Length}
\label{app:single_gpu_scaling}

\begin{table}[H]
\caption{\textbf{Single-GPU scaling across configurations} (1$\times$A100 80GB, Qwen3-4B). Format: peak memory / throughput. Efficiency values are averaged over three measured runs under the specified hardware stack.}
\label{tab:scaling}
\centering
\footnotesize
\resizebox{\textwidth}{!}{
\begin{tabular}{lcccc}
\toprule
Config & E2E & E2E+GC & LoPT & LoPT+GC \\
\midrule
bs=1, seq=4096 & 61.8\,GB / 2{,}840 & 25.6\,GB / 2{,}307 & 42.2\,GB / 3{,}717 & 25.2\,GB / 3{,}061 \\
bs=2, seq=2048 & 61.2\,GB / 3{,}359 & 25.6\,GB / 2{,}784 & 42.2\,GB / 3{,}898 & 25.2\,GB / 3{,}175 \\
bs=2, seq=4096 & OOM & 35.1\,GB / 2{,}578 & 63.5\,GB / 4{,}360 & 32.9\,GB / 3{,}532 \\
bs=4, seq=2048 & OOM & 35.1\,GB / 3{,}145 & 63.5\,GB / 4{,}607 & 32.9\,GB / 3{,}715 \\
bs=8, seq=1024 & OOM & 35.1\,GB / 3{,}528 & 63.5\,GB / 4{,}753 & 32.9\,GB / 3{,}825 \\
\bottomrule
\end{tabular}
}
\end{table}

Table~\ref{tab:scaling} shows that LoPT can make larger single-GPU configurations feasible when E2E runs out of memory. LoPT is faster than E2E+GC across the measured configurations, while LoPT+GC gives the strongest memory reduction. This is the practical distinction between the two mechanisms: LoPT shortens the task-backward graph, whereas checkpointing trades compute for lower activation storage.

LoPT remains faster than E2E+GC across the tested configurations because it shortens the backward graph without recomputation. LoPT+GC provides the strongest memory reduction.

\subsection{Extended Dual-GPU Configurations}
\label{app:dual_gpu_scaling}

\begin{table}[H]
\caption{\textbf{Extended dual-GPU configurations} (2$\times$A100 80GB, Qwen3-4B). Format: peak memory / throughput. Efficiency values are averaged over three measured runs under the specified hardware stack.}
\label{tab:dual_extended}
\centering
\footnotesize
\resizebox{\textwidth}{!}{
\begin{tabular}{lcccc}
\toprule
Config & E2E (DDP) & E2E+GC (DDP) & LoPT (PP) & LoPT+GC (PP) \\
\midrule
bs=1, seq=2048 & 49.6\,GB / 2{,}735 & 31.5\,GB / 2{,}256 & 23.2\,GB / 3{,}343 & \textbf{13.6}\,GB / 2{,}769 \\
bs=2, seq=2048 & 68.7\,GB / 3{,}350 & 33.0\,GB / 2{,}761 & 33.8\,GB / 4{,}294 & \textbf{16.8}\,GB / 3{,}453 \\
bs=2, seq=4096 & OOM & 42.6\,GB / 2{,}568 & 55.1\,GB / 4{,}607 & \textbf{24.6}\,GB / 3{,}702 \\
bs=4, seq=2048 & OOM & 42.6\,GB / 3{,}144 & 37.9\,GB / 4{,}563 & \textbf{20.9}\,GB / 3{,}648 \\
bs=8, seq=2048 & OOM & 61.8\,GB / 3{,}334 & 37.9\,GB / 4{,}710 & \textbf{20.9}\,GB / 3{,}712 \\
\bottomrule
\end{tabular}
}
\end{table}

The extended dual-GPU results show the same pattern at larger measured batch/sequence settings. E2E runs out of memory in several configurations where LoPT or LoPT+GC remains feasible. The LoPT+GC column is especially useful under tight memory budgets, while LoPT without GC is preferable when throughput is the priority.

\subsection{Feasible Configurations under Memory Budgets}
\label{app:memory_budget}

\begin{table}[H]
\caption{\textbf{Representative feasible measured configurations under different memory budgets} for the profiled Qwen3-4B setting.}
\label{tab:max_config}
\centering
\small
\resizebox{\textwidth}{!}{%
\begin{tabular}{lcccc}
\toprule
Memory Budget & E2E & E2E+GC & LoPT & LoPT+GC \\
\midrule
24\,GB & infeasible & infeasible & PP: bs=1, seq=2048 & PP: bs=8, seq=2048 \\
40\,GB & infeasible & bs=4, seq=2048 & PP: bs=8, seq=2048 & PP: bs=8, seq=2048 \\
80\,GB & bs=2, seq=2048 & DDP: bs=8, seq=2048 & bs=2, seq=4096 & PP: bs=8, seq=2048 \\
\bottomrule
\end{tabular}%
}
\end{table}

The table above includes only measured configurations. In principle, a larger PP configuration could become feasible under an 80\,GB budget, but we did not measure such a run here, so we exclude it from the tabulated claims.

\FloatBarrier
\section{LoRA Quality Results}
\label{app:lora_results}

This section provides the quality results behind the compatibility discussion in Section~\ref{sec:capability}. LoRA constrains the parameterization of adaptation, while LoPT constrains the reach of task gradients. The comparison therefore asks whether LoPT still helps when trainable updates are already low-rank.

\begin{table}[H]
\caption{\textbf{LoPT+LoRA SFT} (Qwen2.5-7B, MetaMathQA 100K, LoRA rank=32, seven benchmarks, \%). Reported benchmark scores are averaged over three independent runs with different random seeds.}
\label{tab:lora_sft}
\centering
\small
\begin{tabular}{lcccc}
\toprule
Benchmark & Base & E2E+LoRA & LoPT+LoRA & $\Delta$ \\
\midrule
MMLU & 74.17 & 73.82 & \textbf{74.16} & \llneutral{$+0.34$} \\
IFEval & 57.30 & 53.05 & \textbf{55.45} & \llwin{$\mathbf{+2.40}$} \\
ARC-C & 66.72 & 64.68 & \textbf{65.27} & \llneutral{$+0.59$} \\
GSM8K & 83.47 & 84.49 & \textbf{85.13} & \llneutral{$\mathbf{+0.64}$} \\
HellaSwag & 79.35 & 80.65 & \textbf{81.15} & \llneutral{$+0.50$} \\
TruthfulQA & 64.75 & 61.58 & \textbf{61.62} & \llneutral{$+0.04$} \\
Winogrande & 74.51 & 73.88 & \textbf{74.03} & \llneutral{$+0.15$} \\
\midrule
LoPT Wins & & \multicolumn{2}{c}{\textbf{7/7}} & Avg: \llneutral{$+0.67\pp$} \\
\bottomrule
\end{tabular}
\end{table}

\begin{table}[H]
\caption{\textbf{LoPT+LoRA GRPO} (Qwen2.5-7B, GSM8K GRPO, LoRA rank=32, seven benchmarks, \%). Reported benchmark scores are averaged over three independent runs with different random seeds.}
\label{tab:lora_grpo}
\centering
\small
\begin{tabular}{lcccc}
\toprule
Benchmark & Base & E2E+LoRA & LoPT+LoRA & $\Delta$ \\
\midrule
MMLU & 74.17 & 74.13 & \textbf{74.23} & \llneutral{$+0.10$} \\
IFEval & 57.30 & \textbf{59.15} & 58.96 & \llneutral{$-0.19$} \\
ARC-C & 66.72 & 66.04 & \textbf{66.55} & \llneutral{$+0.51$} \\
GSM8K & 83.47 & 88.47 & \textbf{88.92} & \llneutral{$+0.45$} \\
HellaSwag & 79.35 & 81.32 & \textbf{81.34} & \llneutral{$+0.02$} \\
TruthfulQA & 64.75 & \textbf{64.76} & 64.75 & \llneutral{$-0.01$} \\
Winogrande & 74.51 & 73.40 & \textbf{73.56} & \llneutral{$+0.16$} \\
\midrule
LoPT Wins & & \multicolumn{2}{c}{5/7} & Avg: \llneutral{$+0.15\pp$} \\
\bottomrule
\end{tabular}
\end{table}

LoPT+LoRA improves all seven SFT benchmarks over E2E+LoRA, with the largest gain on IFEval. Under GRPO, the average difference is much smaller ($+0.15\pp$), so we read the result as near parity rather than a strong quality gain. Together with the profiling results in Table~\ref{tab:stacking}, these tables show that LoPT can be stacked with LoRA: LoRA controls update rank, while LoPT controls which portion of the model receives task-gradient signals.

\section{32B GRPO Sanity Check}
\label{app:scale32b}

This section reports an additional 32B GRPO sanity check. It is not included in the 17 main configurations because the main quality sweep focuses on 4.0B--8.0B models. We therefore retain it only as a limited near-parity check, not as evidence about scaling behavior.

\begin{table}[H]
\caption{\textbf{Qwen2.5-32B GSM8K GRPO} --- seven-benchmark results (\%). Reported values are averaged over three independent runs with different random seeds. Averages are computed from unrounded lm-eval outputs.}
\label{tab:qwen32b_grpo}
\centering
\small
\begin{tabular}{lcccc}
\toprule
Benchmark & Base & E2E-GRPO & LoPT-GRPO & $\Delta$ \\
\midrule
MMLU & 83.32 & 83.24 & \textbf{83.34} & \llneutral{$+0.10$} \\
IFEval & 65.80 & \textbf{66.17} & 65.51 & \llneutral{$-0.66$} \\
ARC-C & 74.06 & 74.23 & \textbf{74.55} & \llneutral{$+0.32$} \\
GSM8K & 90.05 & 94.47 & \textbf{95.04} & \llneutral{$+0.57$} \\
HellaSwag & 85.90 & 85.87 & \textbf{86.24} & \llneutral{$+0.37$} \\
TruthfulQA & 65.54 & 65.73 & \textbf{65.75} & \llneutral{$+0.02$} \\
Winogrande & 80.35 & 80.27 & \textbf{80.71} & \llneutral{$+0.44$} \\
\midrule
LoPT Wins & & \multicolumn{2}{c}{6/7} & Avg: \llneutral{$+0.17\pp$} \\
\bottomrule
\end{tabular}
\end{table}

At 32B, E2E and LoPT are near parity under GSM8K GRPO. LoPT is higher on 6 of 7 benchmarks, the average difference is $+0.16\pp$, and all benchmark-level differences are within one percentage point. We therefore interpret this result only as a sanity check that the midpoint boundary does not obviously degrade this setting, not as evidence about scaling.

\section{GRPO Implementation Details}
\label{app:grpo_details}

\begin{table}[H]
\caption{\textbf{GRPO hyperparameters and implementation details.}}
\label{tab:grpo_hyperparams}
\centering
\small
\begin{tabular}{ll}
\toprule
Component & Detail \\
\midrule
Optimizer & AdamW \\
Learning rate ($k_1$, $k_2$) & $5 \times 10^{-6}$ \\
Gradient clipping & max norm = 1.0 \\
Training precision & bfloat16 \\
Auxiliary loss weight $\lambda_{\text{aux}}$ & 10.0 \\
Rollout engine & vLLM 0.7.3, tensor parallel \\
Rollout temperature & 0.7 \\
Rollout top-$p$ & 0.95 \\
Max generation length & 1,024 tokens \\
Training truncation & 2,048 tokens \\
Generations per prompt ($G$) & 4 \\
Prompts per step ($B$) & 4 \\
Clip $\varepsilon$ & 0.2 \\
KL penalty / reference model & None \\
Log-probability ratio & Token-level ratio, matching Eq.~\eqref{eq:lopt_grpo} \\
Advantage normalization & Mean/std within each prompt group \\
Per-sample backward & Yes, with immediate graph release \\
Reward (GSM8K) & Exact match of numerical answer \\
Reward (NuminaMath) & Exact match of numerical answer \\
\bottomrule
\end{tabular}
\end{table}

GRPO trains with chat-template formatted prompts, while the main paper evaluates with lm-eval prompts for consistency across SFT and GRPO. The released configs should pin both the training chat template and the lm-eval evaluation prompt templates, since the two protocols answer different questions and should not be numerically mixed. We report the token-level ratio here to match the GRPO objective used in Eq.~\eqref{eq:lopt_grpo}; if a sequence-level ratio is used in a future implementation, the objective and this table should be changed together.

\section{Exact Implementation Details}
\label{app:exact_impl}

This appendix gives the implementation-level gradient-routing structure used by LoPT. The main point is that LoPT uses two separated backward paths with two parameter groups: the auxiliary path updates the first-half block and the reconstruction head, while the task path updates only the second-half block through a detached boundary activation. In the implementation below, \(\theta_1\) denotes the first-half block, \(\theta_2\) denotes the second-half block including the final norm and output head, and \(\phi\) denotes the auxiliary reconstruction head.

\paragraph{LoPT SFT step.}
\begin{center}
\fbox{%
\begin{minipage}{0.96\linewidth}
\small
\textbf{Algorithm A.1: One LoPT SFT update}\vspace{0.25em}
\begin{enumerate}
    \item \textbf{Input:} tokenized batch \(\mathbf{x}\), non-padding prediction mask \(\mathcal{T}\), optimizers \(\mathrm{Opt}_{1,\phi}\) and \(\mathrm{Opt}_{2}\).
    \item Compute embeddings \(\mathbf{h}_0=\texttt{Embed}(\mathbf{x})\).
    \item Forward through the first-half block: \(\mathbf{h}_1=f_{k_1}(\mathbf{h}_0;\theta_1)\).
    \item Compute the reconstruction target with no gradient: \(\mathbf{z}_0=\operatorname{sg}(\mathbf{h}_0)\).
    \item Compute \(\mathcal{L}_{\mathrm{aux}}=\frac{1}{|\mathcal{M}|d}\sum_{(b,\ell)\in\mathcal{M}}\|g_\phi(\mathbf{h}_{1,b,\ell})-\mathbf{z}_{0,b,\ell}\|_2^2\).
    \item Zero gradients for \(\theta_1\) and \(\phi\); backpropagate \(\lambda_{\mathrm{aux}}\mathcal{L}_{\mathrm{aux}}\); clip gradients if enabled; step \(\mathrm{Opt}_{1,\phi}\).
    \item Detach the boundary activation: \(\widehat{\mathbf{h}}_1=\operatorname{sg}(\mathbf{h}_1)\).
    \item Forward only the second-half task branch: \(\mathbf{y}=f_{k_2}(\widehat{\mathbf{h}}_1;\theta_2)\).
    \item Compute the standard SFT loss on \(\mathcal{T}\), zero gradients for \(\theta_2\), backpropagate \(\mathcal{L}_{\mathrm{SFT}}\), clip gradients if enabled, and step \(\mathrm{Opt}_{2}\).

\end{enumerate}
\end{minipage}%
}
\end{center}

\noindent The task branch uses the numerical boundary output produced by the first-half block in the same update, but \(\widehat{\mathbf{h}}_1\) is detached, so \(\mathcal{L}_{\mathrm{SFT}}\) does not create gradients for \(\theta_1\). The auxiliary reconstruction loss provides the local update for \((\theta_1,\phi)\), while the SFT loss updates \(\theta_2\) through the detached midpoint boundary.

\paragraph{LoPT GRPO policy-update step.}
\begin{center}
\fbox{%
\begin{minipage}{0.96\linewidth}
\small
\textbf{Algorithm A.2: One LoPT GRPO policy-update step}\vspace{0.25em}
\begin{enumerate}
    \item \textbf{Rollout:} sample \(G\) responses per prompt using the current full policy. This rollout path is unchanged by LoPT and does not use the auxiliary head.
    \item Compute rewards and group-normalized advantages \(\{\widehat{A}_i\}_{i=1}^{G}\). Keep old-policy log-probabilities fixed for the clipped-ratio objective.
    \item For each response prefix, compute \(\mathbf{h}_{0,i,t}\), run \(k_1\), and obtain \(\mathbf{h}_{1,i,t}=f_{k_1}(\mathbf{h}_{0,i,t};\theta_1)\).
    \item As in SFT, first backpropagate \(\lambda_{\mathrm{aux}}\mathcal{L}_{\mathrm{aux}}\) through \(k_1\) and \(g_\phi\), then step \(\mathrm{Opt}_{1,\phi}\).
    \item Detach the boundary states as
    \(\widehat{\mathbf{h}}_{1,i,t}=\operatorname{sg}(\mathbf{h}_{1,i,t})\),
    and evaluate current-policy token log-probabilities through \(k_2\).
    \item Compute token-level ratios \(r_{i,t}\), the clipped GRPO loss, and backpropagate this loss only through \(\theta_2\).
    \item Clip gradients if enabled and step \(\mathrm{Opt}_{2}\). No task-gradient update is applied to \(\theta_1\).
\end{enumerate}
\end{minipage}%
}
\end{center}

\noindent LoPT does not modify the reward function, group construction, advantage normalization, rollout sampling parameters, clipping coefficient, or evaluation protocol. It changes only the policy-update backward path by detaching the midpoint boundary before the task loss is evaluated through \(k_2\).

\paragraph{Gradient-isolation checks.}
In our implementation, the stop-gradient boundary is enforced by applying \texttt{detach()} to the boundary activation consumed by the task branch. The auxiliary target \(\operatorname{sg}(\mathbf{h}_0)\) is also detached, so the reconstruction loss updates \(k_1\) and \(g_\phi\) but not the embedding target through the reconstruction path. A simple implementation check is that, after the task-loss backward call and before the \(\theta_2\) optimizer step, all parameters assigned only to \(k_1\) should have zero or \texttt{None} task gradients. Conversely, \(k_2\) should receive no auxiliary-loss gradients.

\paragraph{Tied embeddings and output heads.}
For models with tied input embeddings and output heads, LoPT treats the two roles as distinct parameter tensors when defining the gradient boundary. This prevents task gradients from reaching the input embedding through a tied output head and bypassing the midpoint stop-gradient boundary. When a model family does not tie these weights, no special handling is needed.

\section{Metric Conventions}
\label{app:metrics}

All main numerical results are extracted from lm-eval-harness JSON outputs. Table~\ref{tab:metric_mapping} lists the metric keys used for the seven-benchmark suite.

\begin{table}[H]
\caption{\textbf{Benchmark-to-metric mapping.}}
\label{tab:metric_mapping}
\centering\small
\begin{tabular}{lll}
\toprule
Benchmark & lm-eval Metric Key & Notes \\
\midrule
MMLU & \texttt{acc,none} & 5-shot \\
IFEval & \texttt{prompt\_level\_strict\_acc,none} & 0-shot, strict \\
ARC-Challenge & \texttt{acc\_norm,none} & 25-shot, length-normalized \\
GSM8K & \texttt{exact\_match,strict-match} & 4-shot, strict match \\
HellaSwag & \texttt{acc\_norm,none} & 5-shot, length-normalized \\
TruthfulQA MC2 & \texttt{mc2,none} & 0-shot \\
Winogrande & \texttt{acc,none} & 5-shot \\
\bottomrule
\end{tabular}
\end{table}

\section{Compute Budget}
\label{app:compute}

\begin{table}[H]
\caption{\textbf{Compute budget breakdown.} The compute budget includes all three-seed training runs. Profiling-only tables report repeated measurements under the specified hardware stack.}
\label{tab:compute}
\centering
\small
\begin{tabular}{lc}
\toprule
Experiment Set & GPU-Hours (A100) \\
\midrule
SFT profiling and efficiency analysis & $\sim$16 \\
SFT training (12 configurations $\times$ 2 methods $\times$ 3 seeds) & $\sim$1{,}440 \\
SFT evaluation (72 trained checkpoints $\times$ 7 benchmarks) & $\sim$288 \\
GRPO training (main GRPO, LoRA-GRPO, and diagnostics) & $\sim$1{,}440 \\
GRPO evaluation & $\sim$144 \\
Framework stacking profiling & $\sim$72 \\
\midrule
\textbf{Total} & $\mathbf{\sim 3{,}400}$ \\
\bottomrule
\end{tabular}
\end{table}

\section{Auxiliary Network Architecture Details}
\label{app:aux_net}

The reconstruction head is a lightweight bottleneck MLP:
\[
\text{LayerNorm}(d) \rightarrow \text{Linear}(d,\frac{d}{4}) \rightarrow \text{GELU} \rightarrow \text{Linear}(\frac{d}{4},d).
\]
It is applied only to the boundary representation during training and is discarded at inference. Thus, LoPT adds no inference-time modules or decoding-time computation.

\begin{table}[H]
\caption{\textbf{Auxiliary network parameter counts.}}
\label{tab:aux_params_all}
\centering
\small
\begin{tabular}{lccc}
\toprule
Model & Hidden size $d$ & Bottleneck $\frac{d}{4}$ & Aux parameters \\
\midrule
Qwen3-4B & 2560 & 640 & 3.29M \\
Qwen2.5-7B & 3584 & 896 & 6.44M \\
Llama-3.1-8B & 4096 & 1024 & 8.40M \\
Qwen2.5-32B & 5120 & 1280 & 13.12M \\
\bottomrule
\end{tabular}
\end{table}

The auxiliary head is small compared with the base LLM. For example, Qwen3-4B uses 3.29M auxiliary parameters, about 0.08\% of the model size. We initialize the two linear layers with Xavier uniform initialization using gain 0.1 and initialize biases to zero. Since the head is removed after training, these parameters do not affect inference memory or latency.

\section{Theoretical Analysis}
\label{app:theory}

This section provides a theoretical view of why limiting task-gradient
reach can improve the quality--efficiency trade-off in LoPT. The analysis
does not claim that LoPT optimizes the same stationary condition as
standard end-to-end post-training. In fact, LoPT deliberately introduces
a stop-gradient bias by removing the task-gradient component that would
otherwise update the first-half block. Our goal is therefore more
specific: we characterize how this biased gradient field localizes
task-driven adaptation, bounds first-half-block drift, preserves
information through feature reconstruction, and yields standard
convergence behavior for the separated LoPT update paths.

\subsection{Notation and the LoPT Gradient Field}
\label{app:theory_setup}

Let \(u=(\theta_1,\phi)\) collect the first-half-block parameters and the
auxiliary reconstruction-head parameters, and let \(v=\theta_2\) denote
the second-half-block parameters. For a training example with embedding
state \(\mathbf{h}_0\), define the boundary representation as
\begin{equation}
\begin{aligned}
    \mathbf{z}
    &=
    f_{k_1}(\mathbf{h}_0;\theta_1), \\
    \widehat{\mathbf{z}}
    &=
    \operatorname{sg}(\mathbf{z}), \\
    \mathbf{y}
    &=
    f_{k_2}(\widehat{\mathbf{z}};\theta_2).
\end{aligned}
\label{eq:theory_forward}
\end{equation}

Let \(\ell_{\mathrm{task}}\) denote the SFT or GRPO policy-update loss on
one mini-batch. The task loss evaluated through the LoPT forward path is
written as
\begin{equation}
    \mathcal{T}(\theta_1,\theta_2)
    =
    \ell_{\mathrm{task}}
    \left(
    f_{k_2}
    \left(
    f_{k_1}(\mathbf{h}_0;\theta_1);
    \theta_2
    \right)
    \right).
\label{eq:task_true_value}
\end{equation}

As a scalar function, \(\mathcal{T}\) depends on both \(\theta_1\) and
\(\theta_2\). However, LoPT uses the stop-gradient version of its
backward path:
\begin{equation}
\begin{aligned}
    \nabla_{\theta_1}^{\mathrm{LoPT}}\mathcal{T}
    &=
    \mathbf{0}, \\
    \nabla_{\theta_2}^{\mathrm{LoPT}}\mathcal{T}
    &=
    \nabla_{\theta_2}\mathcal{T}.
\end{aligned}
\label{eq:lopt_task_grad_field}
\end{equation}

The auxiliary reconstruction objective is
\begin{equation}
    \mathcal{A}(\theta_1,\phi)
    =
    \mathbb{E}
    \left[
    \frac{1}{d}
    \left\|
    g_\phi
    \left(
    f_{k_1}(\mathbf{h}_0;\theta_1)
    \right)
    -
    \operatorname{sg}(\mathbf{h}_0)
    \right\|_2^2
    \right],
\label{eq:theory_aux}
\end{equation}
where the expectation is over the training mini-batch distribution and
\(d\) is the hidden dimension. The LoPT update field is therefore
\begin{equation}
\begin{aligned}
    G_{\mathrm{LoPT}}(u,v)
    =
    \left(
    \lambda_{\mathrm{aux}}
    \nabla_{\theta_1}\mathcal{A},
    \lambda_{\mathrm{aux}}
    \nabla_{\phi}\mathcal{A},
    \nabla_{\theta_2}\mathcal{T}
    \right).
\end{aligned}
\label{eq:lopt_gradient_field}
\end{equation}

The missing first-half task-gradient component is the stop-gradient bias:
\begin{equation}
\begin{aligned}
    B_{\mathrm{sg}}(\theta_1,\theta_2)
    &=
    \nabla_{\theta_1}\mathcal{T}(\theta_1,\theta_2)
    -
    \nabla_{\theta_1}^{\mathrm{LoPT}}\mathcal{T}(\theta_1,\theta_2) \\
    &=
    \nabla_{\theta_1}\mathcal{T}(\theta_1,\theta_2).
\end{aligned}
\label{eq:stop_gradient_bias}
\end{equation}

This term is not an approximation error introduced accidentally. It is
the mechanism by which LoPT prevents narrow task gradients from directly
reshaping the first-half block.

\subsection{Local Assumptions Used in the Analysis}
\label{app:theory_assumptions}

The following assumptions are local to the training trajectory. We do not
assume convexity, global smoothness, or that LoPT dominates E2E for every
possible task.

\begin{assumption}[Trajectory-local smoothness]
\label{assump:smoothness}
There exists a compact neighborhood \(\mathcal{U}\) containing the base
checkpoint and the training trajectories of E2E and LoPT. Within
\(\mathcal{U}\), the task loss is smooth in the second-half-block
parameters and the auxiliary loss is smooth in the first-half-block and
auxiliary parameters:
\begin{equation}
\begin{aligned}
    \left\|
    \nabla_{\theta_2}\mathcal{T}(\theta_1,\theta_2)
    -
    \nabla_{\theta_2}\mathcal{T}(\theta_1,\theta_2')
    \right\|_2
    &\le
    L_2
    \left\|
    \theta_2-\theta_2'
    \right\|_2, \\
    \left\|
    \nabla_{u}\mathcal{A}(u)
    -
    \nabla_{u}\mathcal{A}(u')
    \right\|_2
    &\le
    L_A
    \left\|
    u-u'
    \right\|_2.
\end{aligned}
\label{eq:local_smoothness}
\end{equation}
\end{assumption}

\begin{assumption}[Local Jacobian bounds]
\label{assump:jacobian}
Along the same training neighborhood, the Jacobians of the first-half
block, the second-half block with respect to the boundary state, and the
auxiliary reconstruction head are bounded:
\begin{equation}
\begin{aligned}
    \left\|
    \frac{\partial f_{k_1}}{\partial \theta_1}
    \right\|_{\mathrm{op}}
    &\le
    \kappa_1, \\
    \left\|
    \frac{\partial f_{k_2}}{\partial \mathbf{z}}
    \right\|_{\mathrm{op}}
    &\le
    \kappa_2, \\
    \left\|
    \frac{\partial g_{\phi}}{\partial \mathbf{z}}
    \right\|_{\mathrm{op}}
    &\le
    \kappa_g.
\end{aligned}
\label{eq:jacobian_bounds}
\end{equation}
\end{assumption}

\begin{assumption}[Bounded stochastic-gradient variance]
\label{assump:variance}
The stochastic gradients used by LoPT are unbiased for their respective
update fields and have bounded conditional variance:
\begin{equation}
\begin{aligned}
    \mathbb{E}
    \left[
    \widehat{\nabla}_{u}\mathcal{A}(u_t)
    \mid
    \mathcal{F}_t
    \right]
    &=
    \nabla_{u}\mathcal{A}(u_t), \\
    \mathbb{E}
    \left[
    \left\|
    \widehat{\nabla}_{u}\mathcal{A}(u_t)
    -
    \nabla_{u}\mathcal{A}(u_t)
    \right\|_2^2
    \mid
    \mathcal{F}_t
    \right]
    &\le
    \sigma_A^2, \\
    \mathbb{E}
    \left[
    \widehat{\nabla}_{\theta_2}\mathcal{T}(u_{t+1},\theta_{2,t})
    \mid
    \mathcal{F}_{t+1/2}
    \right]
    &=
    \nabla_{\theta_2}\mathcal{T}(u_{t+1},\theta_{2,t}), \\
    \mathbb{E}
    \left[
    \left\|
    \widehat{\nabla}_{\theta_2}\mathcal{T}(u_{t+1},\theta_{2,t})
    -
    \nabla_{\theta_2}\mathcal{T}(u_{t+1},\theta_{2,t})
    \right\|_2^2
    \mid
    \mathcal{F}_{t+1/2}
    \right]
    &\le
    \sigma_2^2.
\end{aligned}
\label{eq:variance_bounds}
\end{equation}
\end{assumption}

\begin{assumption}[Held-out local regularity]
\label{assump:heldout}
Let \(\mathcal{H}(\theta_1,\theta_2)\) be a held-out benchmark loss
evaluated around the base checkpoint. In a local neighborhood of the base
checkpoint, \(\mathcal{H}\) has bounded first-order sensitivity and
second-order curvature:
\begin{equation}
\begin{aligned}
    \left\|
    \nabla_{\theta_1}\mathcal{H}
    \right\|_2
    &\le
    S_1, \\
    \left\|
    \nabla_{\theta_2}\mathcal{H}
    \right\|_2
    &\le
    S_2, \\
    \left\|
    \nabla^2 \mathcal{H}
    \right\|_{\mathrm{op}}
    &\le
    \beta_H.
\end{aligned}
\label{eq:heldout_sensitivity}
\end{equation}
\end{assumption}

Assumption~\ref{assump:heldout} is not a claim that benchmark behavior is
globally smooth. It is only a local Taylor approximation around the base
and tuned checkpoints, used to formalize how parameter drift can affect
held-out performance.

\subsection{Task-Gradient Reach and First-Half-Block Drift}
\label{app:theory_gradient_reach}

We first compare the gradient reaching the first-half block under E2E and
LoPT.

\begin{lemma}[Task-gradient reach under E2E]
\label{lem:e2e_reach}
For one mini-batch, let
\begin{equation}
    \mathbf{g}_{\mathbf{y}}
    =
    \nabla_{\mathbf{y}}\ell_{\mathrm{task}}(\mathbf{y}).
\label{eq:output_gradient}
\end{equation}
Under E2E post-training, the task-gradient component reaching the
first-half block is
\begin{equation}
    \nabla_{\theta_1}\mathcal{T}
    =
    \left(
    \frac{\partial f_{k_1}}{\partial \theta_1}
    \right)^{\top}
    \left(
    \frac{\partial f_{k_2}}{\partial \mathbf{z}}
    \right)^{\top}
    \mathbf{g}_{\mathbf{y}}.
\label{eq:e2e_theta1_gradient}
\end{equation}
Under Assumption~\ref{assump:jacobian}, its norm is bounded by
\begin{equation}
    \left\|
    \nabla_{\theta_1}\mathcal{T}
    \right\|_2
    \le
    \kappa_1
    \kappa_2
    \left\|
    \mathbf{g}_{\mathbf{y}}
    \right\|_2.
\label{eq:e2e_reach_bound}
\end{equation}
\end{lemma}

\begin{proof}
The expression follows from the chain rule through \(k_2\) and then
\(k_1\). The norm bound follows from submultiplicativity of the operator
norm:
\begin{equation}
\begin{aligned}
    \left\|
    \nabla_{\theta_1}\mathcal{T}
    \right\|_2
    &=
    \left\|
    \left(
    \frac{\partial f_{k_1}}{\partial \theta_1}
    \right)^{\top}
    \left(
    \frac{\partial f_{k_2}}{\partial \mathbf{z}}
    \right)^{\top}
    \mathbf{g}_{\mathbf{y}}
    \right\|_2 \\
    &\le
    \left\|
    \frac{\partial f_{k_1}}{\partial \theta_1}
    \right\|_{\mathrm{op}}
    \left\|
    \frac{\partial f_{k_2}}{\partial \mathbf{z}}
    \right\|_{\mathrm{op}}
    \left\|
    \mathbf{g}_{\mathbf{y}}
    \right\|_2 \\
    &\le
    \kappa_1
    \kappa_2
    \left\|
    \mathbf{g}_{\mathbf{y}}
    \right\|_2.
\end{aligned}
\label{eq:e2e_reach_proof}
\end{equation}
\end{proof}

Under LoPT, the task part of this gradient is removed. The first-half
block instead receives the reconstruction gradient.

\begin{lemma}[Reconstruction-gradient reach under LoPT]
\label{lem:recon_reach}
Let the reconstruction residual be
\begin{equation}
    \mathbf{r}
    =
    g_\phi
    \left(
    f_{k_1}(\mathbf{h}_0;\theta_1)
    \right)
    -
    \mathbf{h}_0.
\label{eq:reconstruction_residual}
\end{equation}
Ignoring the scalar normalization constants, the first-half-block update
under LoPT is proportional to
\begin{equation}
    \nabla_{\theta_1}\mathcal{A}
    =
    \left(
    \frac{\partial f_{k_1}}{\partial \theta_1}
    \right)^{\top}
    \left(
    \frac{\partial g_\phi}{\partial \mathbf{z}}
    \right)^{\top}
    \mathbf{r}.
\label{eq:recon_theta1_gradient}
\end{equation}
Under Assumption~\ref{assump:jacobian}, its norm is bounded by
\begin{equation}
    \left\|
    \nabla_{\theta_1}\mathcal{A}
    \right\|_2
    \le
    \kappa_1
    \kappa_g
    \left\|
    \mathbf{r}
    \right\|_2.
\label{eq:recon_reach_bound}
\end{equation}
\end{lemma}

\begin{proof}
The result follows from applying the chain rule to the reconstruction
loss and using the same operator-norm argument as in
Lemma~\ref{lem:e2e_reach}:
\begin{equation}
\begin{aligned}
    \left\|
    \nabla_{\theta_1}\mathcal{A}
    \right\|_2
    &=
    \left\|
    \left(
    \frac{\partial f_{k_1}}{\partial \theta_1}
    \right)^{\top}
    \left(
    \frac{\partial g_\phi}{\partial \mathbf{z}}
    \right)^{\top}
    \mathbf{r}
    \right\|_2 \\
    &\le
    \left\|
    \frac{\partial f_{k_1}}{\partial \theta_1}
    \right\|_{\mathrm{op}}
    \left\|
    \frac{\partial g_\phi}{\partial \mathbf{z}}
    \right\|_{\mathrm{op}}
    \left\|
    \mathbf{r}
    \right\|_2 \\
    &\le
    \kappa_1
    \kappa_g
    \left\|
    \mathbf{r}
    \right\|_2.
\end{aligned}
\label{eq:recon_reach_proof}
\end{equation}
\end{proof}

The two bounds highlight the mechanism. E2E sends a task-gradient vector
through the entire second-half Jacobian before reaching the first-half
block. LoPT removes this path and replaces it with a local reconstruction
gradient whose magnitude is controlled by the reconstruction residual.

\begin{proposition}[First-half drift under LoPT]
\label{prop:first_half_drift}
Suppose LoPT is run for \(T\) steps with first-half step size
\(\eta_1\). Let
\begin{equation}
    D_1^{\mathrm{LoPT}}(T)
    =
    \left\|
    \theta_{1,T}
    -
    \theta_{1,0}
    \right\|_2
\label{eq:lopt_drift_def}
\end{equation}
be the accumulated first-half-block parameter drift. Then
\begin{equation}
    D_1^{\mathrm{LoPT}}(T)
    \le
    \eta_1
    \lambda_{\mathrm{aux}}
    \sum_{t=0}^{T-1}
    \left\|
    \nabla_{\theta_1}\mathcal{A}(\theta_{1,t},\phi_t)
    \right\|_2.
\label{eq:lopt_drift_basic}
\end{equation}
Under Assumption~\ref{assump:jacobian}, this gives
\begin{equation}
    D_1^{\mathrm{LoPT}}(T)
    \le
    \eta_1
    \lambda_{\mathrm{aux}}
    \kappa_1
    \kappa_g
    \sum_{t=0}^{T-1}
    \left\|
    \mathbf{r}_t
    \right\|_2.
\label{eq:lopt_drift_residual_bound}
\end{equation}
\end{proposition}

\begin{proof}
By the LoPT first-half update,
\begin{equation}
    \theta_{1,t+1}
    -
    \theta_{1,t}
    =
    -
    \eta_1
    \lambda_{\mathrm{aux}}
    \nabla_{\theta_1}\mathcal{A}(\theta_{1,t},\phi_t).
\label{eq:lopt_update_increment}
\end{equation}
Summing increments and applying the triangle inequality gives
\begin{equation}
\begin{aligned}
    \left\|
    \theta_{1,T}
    -
    \theta_{1,0}
    \right\|_2
    &=
    \left\|
    \sum_{t=0}^{T-1}
    \left(
    \theta_{1,t+1}
    -
    \theta_{1,t}
    \right)
    \right\|_2 \\
    &\le
    \sum_{t=0}^{T-1}
    \left\|
    \theta_{1,t+1}
    -
    \theta_{1,t}
    \right\|_2 \\
    &=
    \eta_1
    \lambda_{\mathrm{aux}}
    \sum_{t=0}^{T-1}
    \left\|
    \nabla_{\theta_1}\mathcal{A}(\theta_{1,t},\phi_t)
    \right\|_2.
\end{aligned}
\label{eq:lopt_drift_proof}
\end{equation}
Applying Lemma~\ref{lem:recon_reach} yields the residual-based bound.
\end{proof}

This proposition formalizes the drift pattern observed empirically: the
first-half block moves only through the reconstruction residual, not
through the downstream task-gradient vector. When reconstruction error is
small, the first-half block remains close to the base checkpoint while
the second-half block absorbs most task-driven adaptation.

\subsection{Why Feature Reconstruction Rather Than Local NTP}
\label{app:theory_recon_vs_ntp}

A natural alternative is to attach a local language-modeling head to the
first-half block and train it with a next-token-prediction loss. This
appears local, but it reintroduces token-level task supervision into the
first-half block. We make this distinction explicit.

Let \(q_\psi(\cdot \mid \mathbf{z})\) be a local decoder attached to the
boundary representation. For target token \(x^+\), the local NTP loss is
\begin{equation}
    \mathcal{N}(\theta_1,\psi)
    =
    -
    \log
    q_\psi
    \left(
    x^+
    \mid
    f_{k_1}(\mathbf{h}_0;\theta_1)
    \right).
\label{eq:local_ntp_loss}
\end{equation}

Let \(\mathbf{p}_{\psi}\) be the local decoder distribution and
\(\mathbf{e}_{x^+}\) be the one-hot target vector. The first-half
gradient of local NTP takes the form
\begin{equation}
    \nabla_{\theta_1}\mathcal{N}
    =
    \left(
    \frac{\partial f_{k_1}}{\partial \theta_1}
    \right)^{\top}
    \left(
    \frac{\partial q_\psi}{\partial \mathbf{z}}
    \right)^{\top}
    \left(
    \mathbf{p}_{\psi}
    -
    \mathbf{e}_{x^+}
    \right).
\label{eq:local_ntp_gradient}
\end{equation}

In contrast, the reconstruction gradient is
\begin{equation}
    \nabla_{\theta_1}\mathcal{A}
    =
    \left(
    \frac{\partial f_{k_1}}{\partial \theta_1}
    \right)^{\top}
    \left(
    \frac{\partial g_\phi}{\partial \mathbf{z}}
    \right)^{\top}
    \left(
    g_\phi(\mathbf{z})
    -
    \mathbf{h}_0
    \right).
\label{eq:recon_gradient_compare}
\end{equation}

The key difference is the target. Local NTP pushes the first-half block
toward token-predictive features for the post-training distribution,
whereas reconstruction pushes it to retain recoverable information about
the input embedding state. This difference can be stated as an
information-preservation property.

\begin{proposition}[Reconstruction discourages representational collapse]
\label{prop:recon_noncollapse}
Assume the reconstruction head is \(L_g\)-Lipschitz in the boundary
representation:
\begin{equation}
    \left\|
    g_\phi(\mathbf{z})
    -
    g_\phi(\mathbf{z}')
    \right\|_2
    \le
    L_g
    \left\|
    \mathbf{z}
    -
    \mathbf{z}'
    \right\|_2.
\label{eq:decoder_lipschitz}
\end{equation}
Suppose two input embedding states \(\mathbf{h}_0\) and
\(\mathbf{h}_0'\) produce boundary states \(\mathbf{z}\) and
\(\mathbf{z}'\), and the reconstruction errors satisfy
\begin{equation}
\begin{aligned}
    \left\|
    g_\phi(\mathbf{z})
    -
    \mathbf{h}_0
    \right\|_2
    &\le
    \epsilon, \\
    \left\|
    g_\phi(\mathbf{z}')
    -
    \mathbf{h}_0'
    \right\|_2
    &\le
    \epsilon.
\end{aligned}
\label{eq:two_point_recon_error}
\end{equation}
Then the boundary states must satisfy
\begin{equation}
    \left\|
    \mathbf{z}
    -
    \mathbf{z}'
    \right\|_2
    \ge
    \frac{
    \left\|
    \mathbf{h}_0
    -
    \mathbf{h}_0'
    \right\|_2
    -
    2\epsilon
    }{
    L_g
    }.
\label{eq:noncollapse_bound}
\end{equation}
\end{proposition}

\begin{proof}
By the triangle inequality,
\begin{equation}
\begin{aligned}
    \left\|
    \mathbf{h}_0
    -
    \mathbf{h}_0'
    \right\|_2
    &\le
    \left\|
    \mathbf{h}_0
    -
    g_\phi(\mathbf{z})
    \right\|_2
    +
    \left\|
    g_\phi(\mathbf{z})
    -
    g_\phi(\mathbf{z}')
    \right\|_2
    +
    \left\|
    g_\phi(\mathbf{z}')
    -
    \mathbf{h}_0'
    \right\|_2 \\
    &\le
    2\epsilon
    +
    L_g
    \left\|
    \mathbf{z}
    -
    \mathbf{z}'
    \right\|_2.
\end{aligned}
\label{eq:noncollapse_proof}
\end{equation}
Rearranging gives the desired result.
\end{proof}

This proposition does not say that reconstruction perfectly preserves all
pretraining information. It says something narrower and directly tied to
LoPT: if reconstruction error remains small, boundary states cannot
collapse arbitrary differences in the input embedding state unless the
decoder Lipschitz constant becomes very large. Local NTP does not provide
such a constraint; many boundary states may yield the same next-token
distribution while discarding distinctions that the second-half block or
held-out tasks still use.

This explains the empirical local-objective ablation. Pure local NTP
forces the first-half block to solve a token-level task locally, which
reintroduces the very task pressure LoPT tries to isolate. Adding
reconstruction can partially stabilize this pressure, but the local
decoding term still biases \(k_1\) toward task-predictive features. The
default reconstruction objective is therefore better aligned with LoPT's
goal: keep \(k_1\) trainable without turning it into a local task solver.

\subsection{Held-Out Retention as a Drift-Control Consequence}
\label{app:theory_retention}

We next connect controlled first-half drift to held-out benchmark
retention. Let \((\theta_{1,0},\theta_{2,0})\) be the base checkpoint.
For any tuned checkpoint, define first-half and second-half drifts as
\begin{equation}
\begin{aligned}
    D_1
    &=
    \left\|
    \theta_1
    -
    \theta_{1,0}
    \right\|_2, \\
    D_2
    &=
    \left\|
    \theta_2
    -
    \theta_{2,0}
    \right\|_2.
\end{aligned}
\label{eq:block_drifts}
\end{equation}

\begin{proposition}[Local held-out degradation bound]
\label{prop:heldout_bound}
Under Assumption~\ref{assump:heldout}, the held-out loss change from the
base checkpoint satisfies
\begin{equation}
    \mathcal{H}(\theta_1,\theta_2)
    -
    \mathcal{H}(\theta_{1,0},\theta_{2,0})
    \le
    S_1D_1
    +
    S_2D_2
    +
    \frac{\beta_H}{2}
    \left(
    D_1^2
    +
    D_2^2
    \right).
\label{eq:heldout_degradation_bound}
\end{equation}
\end{proposition}

\begin{proof}
Taylor expansion around the base checkpoint gives
\begin{equation}
\begin{aligned}
    \mathcal{H}(\theta_1,\theta_2)
    -
    \mathcal{H}(\theta_{1,0},\theta_{2,0})
    &=
    \left\langle
    \nabla_{\theta_1}\mathcal{H},
    \theta_1-\theta_{1,0}
    \right\rangle
    +
    \left\langle
    \nabla_{\theta_2}\mathcal{H},
    \theta_2-\theta_{2,0}
    \right\rangle
    +
    R_H,
\end{aligned}
\label{eq:heldout_taylor}
\end{equation}
where the second-order remainder is bounded by
\begin{equation}
    R_H
    \le
    \frac{\beta_H}{2}
    \left(
    D_1^2
    +
    D_2^2
    \right).
\label{eq:heldout_remainder}
\end{equation}
Applying Cauchy--Schwarz to the first-order terms and using
Assumption~\ref{assump:heldout} yields Eq.~\eqref{eq:heldout_degradation_bound}.
\end{proof}

The bound is local and does not imply that smaller parameter drift always
improves every benchmark. It says that, for held-out losses sensitive to
the first-half block, reducing \(D_1\) tightens a direct upper bound on
held-out degradation. This matches the empirical observation that LoPT is
most helpful when post-training supervision is narrower than the
evaluated held-out capabilities.

For two methods, denote the drifts by
\(D_1^{\mathrm{LoPT}},D_2^{\mathrm{LoPT}}\) and
\(D_1^{\mathrm{E2E}},D_2^{\mathrm{E2E}}\). The difference between the
two upper bounds is
\begin{equation}
\begin{aligned}
    \mathcal{B}_{\mathrm{LoPT}}
    -
    \mathcal{B}_{\mathrm{E2E}}
    &=
    S_1
    \left(
    D_1^{\mathrm{LoPT}}
    -
    D_1^{\mathrm{E2E}}
    \right)
    +
    S_2
    \left(
    D_2^{\mathrm{LoPT}}
    -
    D_2^{\mathrm{E2E}}
    \right) \\
    &\quad
    +
    \frac{\beta_H}{2}
    \left[
    \left(
    D_1^{\mathrm{LoPT}}
    \right)^2
    -
    \left(
    D_1^{\mathrm{E2E}}
    \right)^2
    +
    \left(
    D_2^{\mathrm{LoPT}}
    \right)^2
    -
    \left(
    D_2^{\mathrm{E2E}}
    \right)^2
    \right].
\end{aligned}
\label{eq:bound_difference}
\end{equation}

Thus, when the second-half drift is comparable but the first-half drift
is much smaller, LoPT obtains a strictly tighter held-out degradation
bound. This is precisely the parameter-drift pattern encouraged by the
stop-gradient boundary.

\subsection{Task Performance and Split Sufficiency}
\label{app:theory_split_sufficiency}

LoPT intentionally restricts the task objective to \(k_2\). This can only
maintain task performance if the second-half block has enough capacity to
absorb the desired adaptation given the boundary representation produced
by \(k_1\). We state this as an explicit, task-dependent condition rather
than a universal assumption.

For a split point \(s\), define the best achievable task loss when the
first-half block is held near the LoPT trajectory:
\begin{equation}
    \mathcal{T}_{s}^{\star}
    =
    \inf_{\theta_2}
    \mathcal{T}
    \left(
    \theta_{1}^{\mathrm{LoPT}}(s),
    \theta_2
    \right).
\label{eq:split_best_task}
\end{equation}

Let the E2E reference optimum within the same local training budget be
\begin{equation}
    \mathcal{T}_{\mathrm{E2E}}^{\star}
    =
    \inf_{\theta_1,\theta_2}
    \mathcal{T}(\theta_1,\theta_2).
\label{eq:e2e_best_task}
\end{equation}

The split-induced approximation gap is
\begin{equation}
    \varepsilon_{\mathrm{split}}(s)
    =
    \mathcal{T}_{s}^{\star}
    -
    \mathcal{T}_{\mathrm{E2E}}^{\star}.
\label{eq:split_gap}
\end{equation}

This quantity formalizes the adaptation cost of limiting task-gradient
reach. A split is effective when \(\varepsilon_{\mathrm{split}}(s)\) is
small; over-partitioning can increase this gap by leaving too little
depth directly optimized by the task objective.

\begin{proposition}[Quality trade-off controlled by split gap]
\label{prop:split_tradeoff}
Suppose LoPT optimization reaches a second-half-block task loss within
\(\varepsilon_{\mathrm{opt}}\) of the best loss achievable for the chosen
split. Then
\begin{equation}
    \mathcal{T}
    \left(
    \theta_{1}^{\mathrm{LoPT}}(s),
    \theta_{2}^{\mathrm{LoPT}}(s)
    \right)
    -
    \mathcal{T}_{\mathrm{E2E}}^{\star}
    \le
    \varepsilon_{\mathrm{split}}(s)
    +
    \varepsilon_{\mathrm{opt}}.
\label{eq:split_tradeoff_bound}
\end{equation}
\end{proposition}

\begin{proof}
By definition of \(\varepsilon_{\mathrm{opt}}\),
\begin{equation}
    \mathcal{T}
    \left(
    \theta_{1}^{\mathrm{LoPT}}(s),
    \theta_{2}^{\mathrm{LoPT}}(s)
    \right)
    \le
    \mathcal{T}_{s}^{\star}
    +
    \varepsilon_{\mathrm{opt}}.
\label{eq:split_opt_condition}
\end{equation}
Subtracting \(\mathcal{T}_{\mathrm{E2E}}^{\star}\) and using
Eq.~\eqref{eq:split_gap} gives
Eq.~\eqref{eq:split_tradeoff_bound}.
\end{proof}

This proposition explains why more boundaries need not improve LoPT. A
shorter task-gradient reach can further protect held-out capabilities by
reducing first-half or middle-block drift, but it can also increase
\(\varepsilon_{\mathrm{split}}\) by making the task-adapted region too
small. The midpoint split is therefore a compromise: it limits
unnecessary task-gradient reach while retaining enough second-half depth
for task adaptation.

\subsection{Convergence of the Separated LoPT Updates}
\label{app:theory_convergence}

We now analyze the optimization behavior of LoPT. Let
\(u_t=(\theta_{1,t},\phi_t)\) and \(v_t=\theta_{2,t}\). One implemented
LoPT step can be written as
\begin{equation}
\begin{aligned}
    u_{t+1}
    &=
    u_t
    -
    \eta_1
    \lambda_{\mathrm{aux}}
    \widehat{\nabla}_{u}\mathcal{A}(u_t), \\
    v_{t+1}
    &=
    v_t
    -
    \eta_2
    \widehat{\nabla}_{v}\mathcal{T}(u_{t+1},v_t).
\end{aligned}
\label{eq:lopt_two_step_update}
\end{equation}

The second line uses \(u_{t+1}\) as a block-coordinate analysis device:
it tracks how the auxiliary update to the first-half block perturbs the
task-facing interface before the next second-half update. This notation
is used for the convergence argument and does not require an additional
first-half forward pass in the implementation.

\begin{lemma}[Auxiliary descent]
\label{lem:aux_descent}
Under Assumptions~\ref{assump:smoothness} and~\ref{assump:variance}, if
\(\eta_1\lambda_{\mathrm{aux}}\le 1/L_A\), then
\begin{equation}
\begin{aligned}
    \mathbb{E}
    \left[
    \mathcal{A}(u_{t+1})
    \mid
    \mathcal{F}_t
    \right]
    &\le
    \mathcal{A}(u_t)
    -
    \frac{
    \eta_1
    \lambda_{\mathrm{aux}}
    }{2}
    \left\|
    \nabla_u \mathcal{A}(u_t)
    \right\|_2^2
    +
    \frac{
    L_A
    \eta_1^2
    \lambda_{\mathrm{aux}}^2
    }{2}
    \sigma_A^2.
\end{aligned}
\label{eq:aux_descent}
\end{equation}
\end{lemma}

\begin{proof}
By smoothness of \(\mathcal{A}\),
\begin{equation}
\begin{aligned}
    \mathcal{A}(u_{t+1})
    &\le
    \mathcal{A}(u_t)
    +
    \left\langle
    \nabla_u\mathcal{A}(u_t),
    u_{t+1}-u_t
    \right\rangle
    +
    \frac{L_A}{2}
    \left\|
    u_{t+1}-u_t
    \right\|_2^2.
\end{aligned}
\label{eq:aux_smooth_step}
\end{equation}
Substituting the update and taking conditional expectation gives
\begin{equation}
\begin{aligned}
    \mathbb{E}
    \left[
    \mathcal{A}(u_{t+1})
    \mid
    \mathcal{F}_t
    \right]
    &\le
    \mathcal{A}(u_t)
    -
    \eta_1
    \lambda_{\mathrm{aux}}
    \left\|
    \nabla_u\mathcal{A}(u_t)
    \right\|_2^2 \\
    &\quad
    +
    \frac{
    L_A
    \eta_1^2
    \lambda_{\mathrm{aux}}^2
    }{2}
    \left(
    \left\|
    \nabla_u\mathcal{A}(u_t)
    \right\|_2^2
    +
    \sigma_A^2
    \right).
\end{aligned}
\label{eq:aux_descent_intermediate}
\end{equation}
The step-size condition yields
Eq.~\eqref{eq:aux_descent}.
\end{proof}

\begin{corollary}[Average auxiliary stationarity]
\label{cor:aux_stationarity}
If \(\mathcal{A}\) is lower bounded by \(\mathcal{A}_{\inf}\), then for
\(T\) LoPT steps,
\begin{equation}
    \frac{1}{T}
    \sum_{t=0}^{T-1}
    \mathbb{E}
    \left[
    \left\|
    \nabla_u \mathcal{A}(u_t)
    \right\|_2^2
    \right]
    \le
    \frac{
    2
    \left(
    \mathcal{A}(u_0)
    -
    \mathcal{A}_{\inf}
    \right)
    }{
    \eta_1
    \lambda_{\mathrm{aux}}
    T
    }
    +
    L_A
    \eta_1
    \lambda_{\mathrm{aux}}
    \sigma_A^2.
\label{eq:aux_stationarity}
\end{equation}
\end{corollary}

The task branch also follows a standard descent argument, except that
the first-half block changes between steps. This creates an interface
drift term, which we explicitly keep.

Define the cumulative task-loss perturbation caused by first-half-block
updates as
\begin{equation}
    \Gamma_T
    =
    \sum_{t=0}^{T-1}
    \mathbb{E}
    \left[
    \left|
    \mathcal{T}(u_{t+1},v_t)
    -
    \mathcal{T}(u_t,v_t)
    \right|
    \right].
\label{eq:interface_drift_term}
\end{equation}

\begin{lemma}[Task-branch descent with interface drift]
\label{lem:task_descent}
Under Assumptions~\ref{assump:smoothness} and~\ref{assump:variance}, if
\(\eta_2\le 1/L_2\) and \(\mathcal{T}\) is lower bounded by
\(\mathcal{T}_{\inf}\), then
\begin{equation}
\begin{aligned}
    \frac{1}{T}
    \sum_{t=0}^{T-1}
    \mathbb{E}
    \left[
    \left\|
    \nabla_v
    \mathcal{T}(u_{t+1},v_t)
    \right\|_2^2
    \right]
    &\le
    \frac{
    2
    \left(
    \mathcal{T}(u_0,v_0)
    -
    \mathcal{T}_{\inf}
    +
    \Gamma_T
    \right)
    }{
    \eta_2 T
    } \\
    &\quad
    +
    L_2
    \eta_2
    \sigma_2^2.
\end{aligned}
\label{eq:task_stationarity_with_gamma}
\end{equation}
\end{lemma}

\begin{proof}
For fixed \(u_{t+1}\), smoothness in \(v\) gives
\begin{equation}
\begin{aligned}
    \mathbb{E}
    \left[
    \mathcal{T}(u_{t+1},v_{t+1})
    \mid
    \mathcal{F}_{t+1/2}
    \right]
    &\le
    \mathcal{T}(u_{t+1},v_t)
    -
    \frac{\eta_2}{2}
    \left\|
    \nabla_v \mathcal{T}(u_{t+1},v_t)
    \right\|_2^2 \\
    &\quad
    +
    \frac{
    L_2
    \eta_2^2
    }{2}
    \sigma_2^2.
\end{aligned}
\label{eq:task_descent_step}
\end{equation}
Rearranging and summing over \(t\) gives
\begin{equation}
\begin{aligned}
    \frac{\eta_2}{2}
    \sum_{t=0}^{T-1}
    \mathbb{E}
    \left[
    \left\|
    \nabla_v \mathcal{T}(u_{t+1},v_t)
    \right\|_2^2
    \right]
    &\le
    \sum_{t=0}^{T-1}
    \mathbb{E}
    \left[
    \mathcal{T}(u_{t+1},v_t)
    -
    \mathcal{T}(u_{t+1},v_{t+1})
    \right] \\
    &\quad
    +
    \frac{
    L_2
    \eta_2^2
    T
    }{2}
    \sigma_2^2.
\end{aligned}
\label{eq:task_sum_before_telescoping}
\end{equation}
For each step,
\begin{equation}
\begin{aligned}
    \mathcal{T}(u_{t+1},v_t)
    -
    \mathcal{T}(u_{t+1},v_{t+1})
    &=
    \mathcal{T}(u_t,v_t)
    -
    \mathcal{T}(u_{t+1},v_{t+1}) \\
    &\quad
    +
    \mathcal{T}(u_{t+1},v_t)
    -
    \mathcal{T}(u_t,v_t).
\end{aligned}
\label{eq:task_telescoping_decomposition}
\end{equation}
Summing the first difference telescopes, while the second difference is
bounded by \(\Gamma_T\). Therefore,
\begin{equation}
\begin{aligned}
    \sum_{t=0}^{T-1}
    \mathbb{E}
    \left[
    \mathcal{T}(u_{t+1},v_t)
    -
    \mathcal{T}(u_{t+1},v_{t+1})
    \right]
    &\le
    \mathcal{T}(u_0,v_0)
    -
    \mathcal{T}_{\inf}
    +
    \Gamma_T.
\end{aligned}
\label{eq:task_telescoping_bound}
\end{equation}
Substitution into Eq.~\eqref{eq:task_sum_before_telescoping} and
division by \(\eta_2 T/2\) proves the result.
\end{proof}

The term \(\Gamma_T\) is the price of changing the first-half block while
the task objective is optimized only through the second-half block. It is
not hidden in the proof; it is the exact interface-drift bias induced by
the separated update.

We can further relate \(\Gamma_T\) to reconstruction-driven movement. If
\(\mathcal{T}\) is locally \(L_z\)-Lipschitz in the boundary state and
\(f_{k_1}\) is locally \(\kappa_u\)-Lipschitz in \(u\), then
\begin{equation}
    \left|
    \mathcal{T}(u_{t+1},v_t)
    -
    \mathcal{T}(u_t,v_t)
    \right|
    \le
    L_z
    \kappa_u
    \left\|
    u_{t+1}
    -
    u_t
    \right\|_2.
\label{eq:task_boundary_lipschitz}
\end{equation}

Using the auxiliary update,
\begin{equation}
    \left\|
    u_{t+1}
    -
    u_t
    \right\|_2
    =
    \eta_1
    \lambda_{\mathrm{aux}}
    \left\|
    \widehat{\nabla}_{u}\mathcal{A}(u_t)
    \right\|_2.
\label{eq:u_step_norm}
\end{equation}

Therefore,
\begin{equation}
    \Gamma_T
    \le
    L_z
    \kappa_u
    \eta_1
    \lambda_{\mathrm{aux}}
    \sum_{t=0}^{T-1}
    \mathbb{E}
    \left[
    \left\|
    \widehat{\nabla}_{u}\mathcal{A}(u_t)
    \right\|_2
    \right].
\label{eq:gamma_bound_aux_gradient}
\end{equation}

This bound is important for LoPT. Feature reconstruction aims to keep
\(\nabla_u\mathcal{A}\) and the resulting boundary movement small once
the first-half block maintains a reconstructable interface. Local NTP, in
contrast, can keep producing token-level gradients whenever the local
decoder is imperfect on the post-training distribution, thereby
increasing the interface-drift term.

\subsection{Memory and Computation Implications}
\label{app:theory_efficiency}

The main theoretical systems benefit of LoPT is a reduction in the peak
lifetime of task-loss activations. Let \(A_1\) denote the activation
memory required by the first-half block and \(A_2\) the activation memory
required by the second-half block for a matched mini-batch. Let
\(M_{\mathrm{state}}\) include parameters, optimizer states, gradients,
and other non-activation memory. A simplified E2E peak-memory model is
\begin{equation}
    M_{\mathrm{E2E}}
    =
    M_{\mathrm{state}}
    +
    A_1
    +
    A_2
    +
    M_{\mathrm{misc}}.
\label{eq:e2e_memory_model}
\end{equation}

LoPT separates the auxiliary graph and the task graph. If the auxiliary
backward is completed before constructing the task-loss graph, then the
two large activation sets do not need to coexist. With auxiliary-head
activation memory \(A_g\) and boundary-buffer memory \(A_{\partial}\),
the LoPT peak memory satisfies the implementation-level bound
\begin{equation}
    M_{\mathrm{LoPT}}
    \le
    M_{\mathrm{state}}
    +
    \max
    \left\{
    A_1 + A_g,
    A_2 + A_{\partial}
    \right\}
    +
    M_{\mathrm{misc}}'.
\label{eq:lopt_memory_model}
\end{equation}

If activations are approximately balanced and \(s=N/2\), then
\(A_1\approx A_2\). Writing \(A_1=A_2=A/2\), we obtain
\begin{equation}
    M_{\mathrm{LoPT}}
    -
    M_{\mathrm{state}}
    \lesssim
    \frac{A}{2}
    +
    \max
    \left\{
    A_g,
    A_{\partial}
    \right\}
    +
    M_{\mathrm{misc}}'.
\label{eq:lopt_midpoint_memory}
\end{equation}

Thus, ignoring small overhead terms, the activation component of peak
memory is reduced from \(A\) to roughly \(A/2\). The realized percentage
saving is smaller because parameters, optimizer states, communication
buffers, and framework overheads are not halved.

The compute comparison is more nuanced. In the implementation used for
our experiments, the first-half block is forwarded once to produce the
midpoint boundary, and the auxiliary and task branches consume this
boundary through different gradient routes. A simplified cost model is
\begin{equation}
\begin{aligned}
    C_{\mathrm{E2E}}
    &=
    C_{F,1}
    +
    C_{F,2}
    +
    C_{B,2}^{\mathrm{task}}
    +
    C_{B,1}^{\mathrm{task}}, \\
    C_{\mathrm{LoPT}}
    &=
    C_{F,1}
    +
    C_{F,2}
    +
    C_{B,2}^{\mathrm{task}}
    +
    C_{B,1}^{\mathrm{aux}}
    +
    C_g.
\end{aligned}
\label{eq:compute_model}
\end{equation}

Therefore,
\begin{equation}
    C_{\mathrm{LoPT}}
    -
    C_{\mathrm{E2E}}
    =
    C_{B,1}^{\mathrm{aux}}
    +
    C_g
    -
    C_{B,1}^{\mathrm{task}}.
\label{eq:compute_difference}
\end{equation}

This expression shows why LoPT's speedup is implementation-dependent.
LoPT always shortens the task-induced backward graph, but it also adds an
auxiliary branch. Throughput improves when the reduced activation
lifetime, shorter task-backward path, and better memory behavior outweigh
the auxiliary-branch overhead. This is consistent with the empirical results, where memory
savings are larger and more systematic than raw throughput gains.

\subsection{Summary of the Theoretical Picture}
\label{app:theory_summary}

The analysis supports the following mechanism.

First, E2E post-training sends the task-gradient vector through the full
second-half Jacobian before updating the first-half block. LoPT removes
this path, and the missing component is exactly the stop-gradient bias in
Eq.~\eqref{eq:stop_gradient_bias}. This bias is intentional: it is the
mathematical operation that limits task-gradient reach.

Second, the first-half block is not frozen. Its drift is controlled by
the reconstruction residual through Eq.~\eqref{eq:lopt_drift_residual_bound}.
This explains why LoPT can keep \(k_1\) close to the base checkpoint
while still allowing small, structured, auxiliary-driven movement.

Third, feature reconstruction is different from local next-token
prediction. Reconstruction encourages the boundary representation to
retain recoverable information about the input embedding state and
discourages collapse according to
Eq.~\eqref{eq:noncollapse_bound}. Local NTP instead injects a local
task-level decoding gradient into \(k_1\), which can conflict with LoPT's
goal of shielding the first-half block from direct task supervision.

Fourth, held-out retention follows from a local drift-control argument.
When held-out losses are sensitive to first-half representations, a
smaller \(D_1\) gives a tighter upper bound on held-out degradation via
Eq.~\eqref{eq:heldout_degradation_bound}. This does not guarantee that
LoPT wins on every benchmark, but it explains why limiting first-half
task-induced drift can improve retention when post-training supervision
is narrow.

Finally, LoPT's optimization converges in the sense appropriate to its
separated gradient field. The auxiliary branch approaches stationarity
for the reconstruction objective, and the task branch approaches
stationarity in the second-half parameters up to the explicit
interface-drift term \(\Gamma_T\). This term is the cost of keeping
\(k_1\) trainable while blocking task gradients, and feature
reconstruction is designed to keep it small.

\clearpage
\bibliographystyle{plainnat}
\bibliography{main}

\end{document}